\documentclass{article}

\usepackage{microtype}
\usepackage{graphicx}
\usepackage{wrapfig}
\usepackage{makecell}
\usepackage{array}
\usepackage{amsfonts}
\usepackage{amsmath}
\usepackage{adjustbox}
\usepackage{hyperref}
\usepackage{url}
\usepackage{enumitem}
\usepackage{comment}
\usepackage{caption}
\usepackage{subcaption}
\usepackage[T1]{fontenc}
\usepackage[font=scriptsize,labelfont=bf]{caption}
\usepackage{mathtools}
\usepackage[para]{footmisc}
\usepackage{multirow}
\usepackage{footnote}
\makesavenoteenv{tabular}
\makesavenoteenv{table}
\usepackage{amsthm}
\usepackage[utf8]{inputenc}
\usepackage[english]{babel}

\usepackage{chngcntr}
\usepackage{apptools}
\AtAppendix{\counterwithin{theorem}{section}}
\newtheorem{theorem}{Theorem}%[section]
\usepackage{relsize}

\usepackage{booktabs} % for professional tables
\newcolumntype{?}{!{\vrule width 1pt}}
\usepackage{hyperref}

\DeclareMathOperator*{\argmax}{arg\,max}

\usepackage[accepted]{icml2019}

\icmltitlerunning{Robust Decision Trees Against Adversarial Examples}

\begin{document}

\twocolumn[
\icmltitle{Robust Decision Trees Against Adversarial Examples}



\icmlsetsymbol{equal}{*}

\begin{icmlauthorlist}
\icmlauthor{Hongge Chen}{mit}
\icmlauthor{Huan Zhang}{ucla}
\icmlauthor{Duane Boning}{mit}
\icmlauthor{Cho-Jui Hsieh}{ucla}

\end{icmlauthorlist}

\icmlaffiliation{mit}{MIT, Cambridge, MA 02139, USA}
\icmlaffiliation{ucla}{UCLA, Los Angeles, CA 90095, USA}

\icmlcorrespondingauthor{Hongge Chen}{chenhg@mit.edu}
\icmlcorrespondingauthor{Huan Zhang}{huan@huan-zhang.com}
%\icmlcorrespondingauthor{Duane Boning}{boning@mtl.mit.edu}
%\icmlcorrespondingauthor{Cho-Jui Hsieh}{chohsieh@cs.ucla.edu}

\icmlkeywords{Robustness, Decision Tree, Ensemble Methods}

\vskip 0.3in
]



\printAffiliationsAndNotice{}  % leave blank if no need to mention equal contribution

\begin{abstract}
Although adversarial examples and model robustness have been extensively studied in the context of linear models and neural networks,
research on this issue in tree-based models and how to make tree-based models robust against adversarial examples is still limited. In this paper, we show that tree based models are also vulnerable to adversarial examples and develop a novel algorithm to learn robust trees. At its core, our method aims to optimize the performance under the worst-case perturbation of input features, which leads to a max-min saddle point problem. 
Incorporating this saddle point objective into the decision tree building procedure is non-trivial due to the discrete nature of trees---a naive approach to finding the best split according to this saddle point objective will take exponential time. 
%we formulate decision tree learning  as optimizing a max-min objective function. 
%Due to the discreteness nature of decision trees, finding the best splits to optimize this max-min objective is non-trivial. 
To make our approach practical and scalable, 
we propose efficient tree building algorithms by approximating the inner minimizer in this saddle point problem, and present efficient implementations for 
classical information gain based trees as well as state-of-the-art 
%we show that, by lower bounding the inner maximizer (which corresponds to adversarial attack), efficient tree building algorithms can be developed for the vanilla information gain based classification trees as well as state-of-the-art 
tree boosting models such as XGBoost. 
Experimental results on real world datasets demonstrate that the proposed algorithms can substantially improve the robustness of tree-based models against adversarial examples. 
%significant robustness improvement u.
%In this paper, we show that tree based models are also vulnerable to adversarial examples, i.e., data that are very close to natural data and yet mislead the model. To address this problem, we introduce a novel max-min optimization based decision tree training framework for better robustness. We present the implementations of our frame work in vanilla information gain based decision trees as well as state-of-the-art large scale tree boosting systems as XGBoost. The experimental result shows significant robustness improvement under blackbox attacks.
\end{abstract}
\section{Introduction}
The discovery of adversarial examples in various deep learning models~\citep{szegedy2013intriguing,kos2018adversarial,cheng2018seq2sick,chen2018attacking,carlini2018audio,huang2017adversarial} has led to extensive studies of deep neural network (DNN) robustness under such maliciously crafted subtle perturbations.  Although deep learning-based model robustness has been well-studied in the recent literature from both attack and defense perspectives, studies on the robustness of tree-based models are quite limited~\citep{papernot2016transferability}.

In our paper, we shed light on the adversarial robustness of an important class of machine learning models --- decision trees. Among machine learning models used in practice, tree-based methods stand out in many applications, with state-of-the-art performance. Tree-based methods have achieved widespread success due to their simplicity, efficiency, interpretability, and scalability on large datasets. They have been suggested as an advantageous alternative to deep learning in some cases~\citep{zhou2017deep}. In this paper, we study the robustness of tree-based models under adversarial attacks, and more importantly, we propose a novel robust training framework for tree-based models. Below we highlight our major contributions:
\begin{itemize}[wide,noitemsep,topsep=0pt]
    \item We study the robustness of decision tree-based machine learning algorithms through the lens of adversarial examples. We study both classical decision trees and state-of-the-art ensemble boosting methods such as XGBoost. We show that, similar to neural networks, tree-based models are also vulnerable to adversarial examples. %These well-crafted examples are very close to natural examples often with changes imperceptible to humans, but that can easily fool existing tree-based models at the testing or deployment stages.
    \item We propose a novel robust decision tree training framework to improve robustness against adversarial examples. This method seeks to optimize the worst case condition by solving a max-min problem. This framework is quite general and can be applied to tree-based models with any score function used to choose splitting thresholds. To the best of our knowledge, this is the first work contributing a general robust decision tree training framework against adversarial examples.
    \item We implement our framework in both classical information gain based classification trees and state-of-the-art large-scale tree boosting systems. To scale up our framework, we make necessary and efficient approximations to handle complex models and real world data sets. Our experimental results show consistent and substantial improvements on adversarial robustness.
\end{itemize}
%This paper is organized as follows. In Section~\ref{sec:related} we summarize related work on tree-based models as well as adversarial examples of DNNs. In Section~\ref{sec:robustness_issues} we review existing attacking methods for tree-based models and present our attack results. In Section~\ref{sec:robust_tree} we introduce our novel robust training method for tree-based model.
%, especially its two implementations on information gain based trees and XGBoost. Section~\ref{sec:exp} presents our experimental results and~Section~\ref{sec:conclusion} concludes the paper.
\section{Related Works}
\label{sec:related}
%We first review key background for this paper. Single and ensemble decision trees are  discussed, followed by review of adversarial attacks and defenses in deep neural networks. 

\subsection{Decision Tree and Gradient Boosted Decision Tree}
Decision tree learning methods are widely used in machine learning and data mining. As considered here, the goal is to create a tree structure with each interior node corresponding to one of the input features. Each interior node has two children, and edges to child nodes represent the split condition for that feature. Each leaf provides a prediction value of the model, given that the input features satisfy the conditions represented by the path from the root to that leaf. 
% The quality of a tree can be measured by a scoring function defined on interior nodes. %However, finding the optimal decision tree is known to be  NP-complete even with a simple scoring function \citep{murthy1998automatic,laurent1976constructing}. Therefore i
In practice, 
% similar to other optimization problems in machine learning, 
decision tree learning algorithms are based on greedy search, which builds a tree starting from its root by making locally optimal decisions at each node. Classical decision tree training recursively chooses features, sets thresholds and splits the examples on a node by maximizing a pre-defined score, such as information gain or Gini impurity.

Decision trees are often used within ensemble methods. %Random forests, for example, are bootstrap aggregating tree ensembles. 
A well-known gradient tree boosting method has been developed by~\citet{friedman2000additive,friedman2001greedy} and \citet{friedman2002stochastic} to allow optimization of an arbitrary differentiable loss function. Later scalable tree boosting systems have been built to handle large datasets. For example, pGBRT~\citep{tyree2011parallel} parallelizes the training procedure by data partitioning for faster and distributed training. XGBoost~\citep{chen2016xgboost} is a prominent tree boosting software framework; in data mining contests, 17 out of 29 published winning solutions at Kaggle’s blog in 2015 used XGBoost in their models. LightGBM~\citep{ke2017lightgbm,zhang2017gpu} is another highly efficient boosting framework that utilizes histograms on data features to significantly speed up training. mGBDT~\citep{multi2018Feng} learns hierarchical representations by stacking multiple layers of gradient boosted decision trees (GBDTs).
Other variants such as extreme multi-label GBDT~\citep{si2017gradient} and cost efficient tree boosting approaches~\cite{guyon2017cost,xu2019gradient} have also been proposed recently. 
%Cost efficient tree boosting approaches that consider test-time budget during training have also been proposed recently by~\citet{guyon2017cost} and~\citet{xu2019gradient}.

\subsection{Adversarial Attack for Deep Neural Networks}
%Adversarial attacks on DNNs have received much attention in the machine learning community due to the potential threat to applications such as self-driving cars and face recognition. 
An adversarial attack is a subtle modification of a benign example. In a successful attack, the classifier will misclassify this modified example, while the original example is correctly classified. %Formally, we are given an example $\mathbf{x}$ belonging to class $C_1$ and it can be correctly classified by the model. An attacker finds an example $\mathbf{x}_{adv}$ such that $\mathbf{x}_{adv}$ and $\mathbf{x}$ are very close to each other (for image data, this means $\mathbf{x}_{adv}$ and $\mathbf{x}$ are indistinguishable by human eyes) but the model will mistakenly labels $\mathbf{x}_{adv}$ as $C_2\neq C_1$.
Such attacks can be roughly divided into two categories, white-box attacks and black-box attacks. White-box attacks assume that the model is fully exposed to the attacker, including parameters and structures, while in black-box attacks, the attacker can query the model but has no (direct) access to any internal information inside the model. 
% Here we review some popular adversarial attacks for deep learning models.
FGSM~\citep{goodfellow2015explaining} is one of the first methods in the white-box attack category. It computes the gradient only once to generate an adversarial example. %$\mathbf{x}_{adv}$,
%\[\mathbf{x}_{adv}\leftarrow \text{clip}[\mathbf{x}_0 + \xi \text{\textbf{sgn}}(\nabla J(\mathbf{x}_0))],\]where $\xi$ is the step size, $\mathrm{\textbf{sgn}}(\nabla J(\mathbf{x}_0))$ is the sign of the gradient of the training loss with respect to $\mathbf{x}_0$, and $\mathrm{clip}(\cdot)$ ensures that $\mathbf{x}_{adv}$ stays in the valid input range.
This method is strengthened as Iterative-FGSM (or I-FGSM)~\citep{kurakin2017adversarial}, which applies FGSM multiple times for a higher attack success rate and smaller distortion. C\&W attack~\citep{carlini2017towards} formulates the attack as an optimization problem with an $\ell_2$ penalization.
%$$\min\limits_{\mathbf{x}}~\lambda f(\mathbf{x})+\|\mathbf{x}-\mathbf{x}_0\|_2^2,~~\mathrm{s.t.}~\mathbf{x}\in [0,1]^d,
%$$
%where $f(\mathbf{x})$ is a loss function to encourage the model to output an incorrect label. 
EAD-L1 attack~\citep{chen2018ead} uses a more general formulation than C\&W attack with elastic-net regularization. To bypass some defenses with obfuscated gradients, the BPDA attack introduced by~\citet{athalye2018obfuscated} is shown to successfully circumvent many defenses.

The white-box setting is often argued as being unrealistic in the literature.  In contrast, several recent works have studied ways to fool the model given only model output scores or probabilities. Methods in~\citet{chen2017zoo} and~\citet{ilyas2017query} are able to craft adversarial examples by making queries to obtain the corresponding probability outputs of the model. A stronger and more general attack has been developed recently by~\citet{cheng2018queryefficient}, which does not rely on the gradient nor the smoothness of model output. This enables attackers to successfully attack models that only output hard labels.
%CW, PGD, challenge in attacking decision trees, cite our previous works

%Robustness of decision tree and GBDT, ICML 16 and Papernot

\subsection{Defenses for Deep Neural Networks}
% As mentioned above, adversarial examples in DNNs are considered to be a great threat. 
% Various defense methods have been proposed, such as feature squeezing~\citep{xu2017feature} and defensive distillation~\citep{papernot2016distillation}.
It is difficult to defend against adversarial examples, especially under strong and adaptive attacks.
Some early methods, including feature squeezing~\citep{xu2017feature} and defensive distillation~\citep{papernot2016distillation} have been proven ineffective against stronger attacks like C\&W. Many recently proposed defense methods are based on obfuscated gradients~\citep{guo2017countering, song2017pixeldefend, buckman2018thermometer, ma2018characterizing, samangouei2018defense} and are already overcome by the aforementioned BPDA attack. 

Adversarial training, first introduced in~\citet{kurakin2017adversarial}, is effective on DNNs against various attacks. In adversarial training, adversarial examples are generated during the training process and are used as training data to increase model robustness. This technique has been formally posed as a min-max robust optimization problem in~\citet{madry2017towards} and has achieved very good performance under adversarial attacks. Several recent work have tried to improve over the original adversarial training formulation~\cite{liu2018adversarial,liu2018adv,zhang2019theoretically}.
There are some other methods in the literature seeking to give provable guarantees on the robustness performance, such as distributional robust optimization~\citep{sinha2018certifying}, convex relaxations~\citep{kolter2017provable,wong2018scaling,wang2018mixtrain} and semidefinite relaxations~\citep{raghunathan2018certified}. Some of these methods can be deployed in medium-sized networks and achieve satisfactory robustness.

However, all of the current defense methods assume the model to be differentiable and use gradient based optimizers, so none of them can be directly applied to decision tree based models, which are discrete and non-differentiable.
%Madry, Zico, min-max, and other unreliably defense, carlini's ICML, trick does not work

\section{Adversarial Examples of Decision Tree Based Models}
\label{sec:robustness_issues}

%\begin{wrapfigure}{r}{0.5\textwidth}

Recent developments in machine learning have resulted in the deployment of large-scale tree boosting systems in critical applications such as fraud and malware detection.
%These applications make security studies of tree-based models increasingly important. 
Unlike deep neural networks (DNNs), tree based models are non-smooth, non-differentiable and sometimes interpretable, which might lead to the belief that they are more robust than DNNs.
However, the experiments in our paper show that similar to DNNs, tree-based models can also be easily compromised by adversarial examples. In this paper, we focus on untargeted attacks, which are considered to be successful as long as the model misclassifies the adversarial examples.

Unlike DNNs, algorithms for crafting adversarial examples for tree-based models are poorly studied. The main reason is that tree-based models are discrete and non-differentiable, thus we cannot use common gradient descent based methods for white-box attack. An early attack algorithm designed for single decision trees has been proposed by \citet{papernot2016transferability}, based on greedy search. To find an adversarial example, this method searches the neighborhood of the leaf which produces the original prediction, and finds another leaf labeled as a different class by considering the path from the original leaf to the target leaf, and changing the feature values accordingly to result in misclassification.

A white-box attack against binary classification tree ensembles has been proposed by~\citet{kantchelian2016evasion}. This method finds the exact smallest distortion (measured by some $\ell_p$ norm) necessary to mislead the model. However, the algorithm relies on Mixed Integer Linear Programming (MILP) and thus can be very time-consuming when attacking large scale tree models as arise in XGBoost. In this paper, we use the $\ell_\infty$ version of Kantchelian's attack as one of our methods to evaluate small and mid-size binary classification model robustness. \citet{kantchelian2016evasion} also introduce a faster approximation to generate adversarial examples using symbolic prediction with $\ell_0$ norm minimization and combine this method into an adversarial training approach. Unfortunately, the demonstrated adversarial training is not very effective; despite increasing model robustness for $\ell_0$ norm perturbations, robustness for $\ell_1$, $\ell_2$ and $\ell_\infty$ norm perturbations are noticeably reduced compared to the naturally (non-robustly) trained model.

%\textcolor{red}{HZ: I think it is better to refer our attack method as a general method that does not rely on the gradient nor the smoothness of output, rather than saying it is a hard-label black-box attack. Change other places accordingly.}

\begin{figure}[ht]
\centering
\begin{tabular}{ccc}

\textbf{Original}&\thead{Adversarial of \\ nat. GBDT}&\thead{Adversarial of \\ rob. GBDT}\\
\begin{subfigure}[t]{0.08\textwidth}
\centering
  \includegraphics[width=\linewidth]{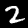}
  \caption{pred.=2~~~~~~~~~~~~~~~~~~}
\end{subfigure}~~~~~&~
\begin{subfigure}[t]{0.08\textwidth}
\centering
  \includegraphics[width=\linewidth]{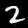}
  \caption{\tabular[t]{@{}l@{}}$\ell_\infty$ dist.$=0.069$\\pred.=8\endtabular}
\end{subfigure}~~~~~&~
\begin{subfigure}[t]{0.08\textwidth}
\centering
  \includegraphics[width=\linewidth]{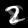}
  \caption{\tabular[t]{@{}l@{}}$\ell_\infty$ dist.$=0.344$\\pred.=8\endtabular}
\end{subfigure}
\\
\begin{subfigure}[t]{0.08\textwidth}
\centering
  \includegraphics[width=\linewidth]{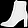}
  \caption{pred.=``Ankle Boot''}
\end{subfigure}~~~~~&~
\begin{subfigure}[t]{0.08\textwidth}
\centering
  \includegraphics[width=\linewidth]{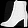}
  \caption{\tabular[t]{@{}l@{}}$\ell_\infty$ dist.$=0.074$\\pred.=``Shirt''\endtabular}
\end{subfigure}~~~~~&~
\begin{subfigure}[t]{0.08\textwidth}
\centering
  \includegraphics[width=\linewidth]{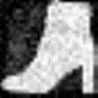}
  \caption{\tabular[t]{@{}l@{}}$\ell_\infty$ dist.$=0.394$\\pred.=``Bag''\endtabular}
\end{subfigure}
\end{tabular}
\caption{MNIST and Fashion-MNIST examples and their adversarial examples found using the untargeted attack proposed by~\citet{cheng2018queryefficient} on 200-tree gradient boosted decision tree (GBDT) models trained using XGBoost with depth=8. Natural GBDT models (nat.) are fooled by small $\ell_\infty$ perturbations (b, e), while our robust (rob.) GBDT models require much larger perturbations (c, f) for successful attacks. For both MNIST and Fashion-MNIST robust models, we use $\epsilon=0.3$ (a robust training hyper-parameter which will be introduced in Section~\ref{sec:robust_tree}).
More examples are shown in the appendix.}
\label{fig:adv_demo}
%\end{wrapfigure}
\end{figure}

In our paper, in addition to Kantchelian attacks we also use a general attack method proposed in~\citet{cheng2018queryefficient} which does not rely on the gradient nor the smoothness of output of a machine learning model. Cheng's attack method has been used to efficiently evaluate the robustness of complex models on large datasets, even under black-box settings. To deal with non-smoothness of model output, this method focuses on the distance between the benign example and the decision boundary, and reformulates the adversarial attack as a minimization problem of this distance. Despite the non-smoothness of model prediction, the distance to decision boundary is usually smooth within a local region, and can be found by binary search on vector length given a direction vector. To minimize this distance without gradient, \citet{cheng2018queryefficient} used a zeroth order optimization algorithm with a randomized gradient-free method. In our paper, we use the $\ell_\infty$ version of Cheng's attack.

Some adversarial examples obtained by this method are shown in Figure~\ref{fig:adv_demo}, where we display results on both MNIST and Fashion-MNIST datasets. The models we test are natural GBDT models trained using XGBoost and our robust GBDT models, each with 200 trees and a tree depth of 8. Cheng's attack is able to craft adversarial examples with very small distortions on natural models; for human eyes, the adversarial distortion added to the natural model's adversarial examples appear as imperceptible noise. We also conduct white-box attacks using the MILP formulation~\citep{kantchelian2016evasion}, which takes much longer time to solve but the $\ell_\infty$ distortion found by MILP is comparable to Cheng's method; see Section~\ref{sec:exp} for more details. In contrast, for our robust GBDT model, the required adversarial example distortions are so large that we can even vaguely see a number~8 in subfigure~(c). The substantial increase in the $\ell_\infty$ distortion required to misclassify as well as the increased visual impact of such distortions shows the effectiveness of our robust decision tree training, which we will introduce in detail next. In the main text, we use the $\ell_\infty$ version of Kantchelian's attack; we present results of $\ell_1$ and $\ell_2$ Kantchelian attacks in the appendix.
\section{Robust Decision Trees}
\label{sec:robust_tree}
\subsection{Intuition}
As shown in Section~\ref{sec:robustness_issues}, tree-based models are vulnerable to adversarial examples. Thus it is necessary to augment the classical natural tree training procedure in order to obtain reliable models robust against adversarial attacks. Our method formulates the process of optimally finding best split threshold in decision tree training as a robust optimization problem. As a conceptual illustration, Figure~\ref{fig:toy} presents a special case where the traditional greedy optimal splitting may yield non-robust models. A horizontal split achieving high accuracy or score on original points may be easily compromised by adversarial perturbations. On the other hand, we are able to select a better vertical split considering possible perturbations in $\ell_\infty$ balls. At a high level, the robust splitting feature and threshold take the distances between data points into account (which is often ignored in most decision tree learning algorithms) and tries to optimize the \emph{worst case} performance under adversarial perturbations. Some recent works in DNNs~\cite{Ilyas2019Adversarial,tsipras2018robustness} divided features into two categories, robust features and non-robust features. In tree-based models, the effect of this dichotomy on the robustness is straight forward, as seen in the two different splits in Figure~\ref{fig:toy} using $x^{(1)}$ (a robust feature) and $x^{(2)}$ (a non-robust feature).

\begin{figure}[tb]
\centering
  \includegraphics[width=0.9\linewidth]{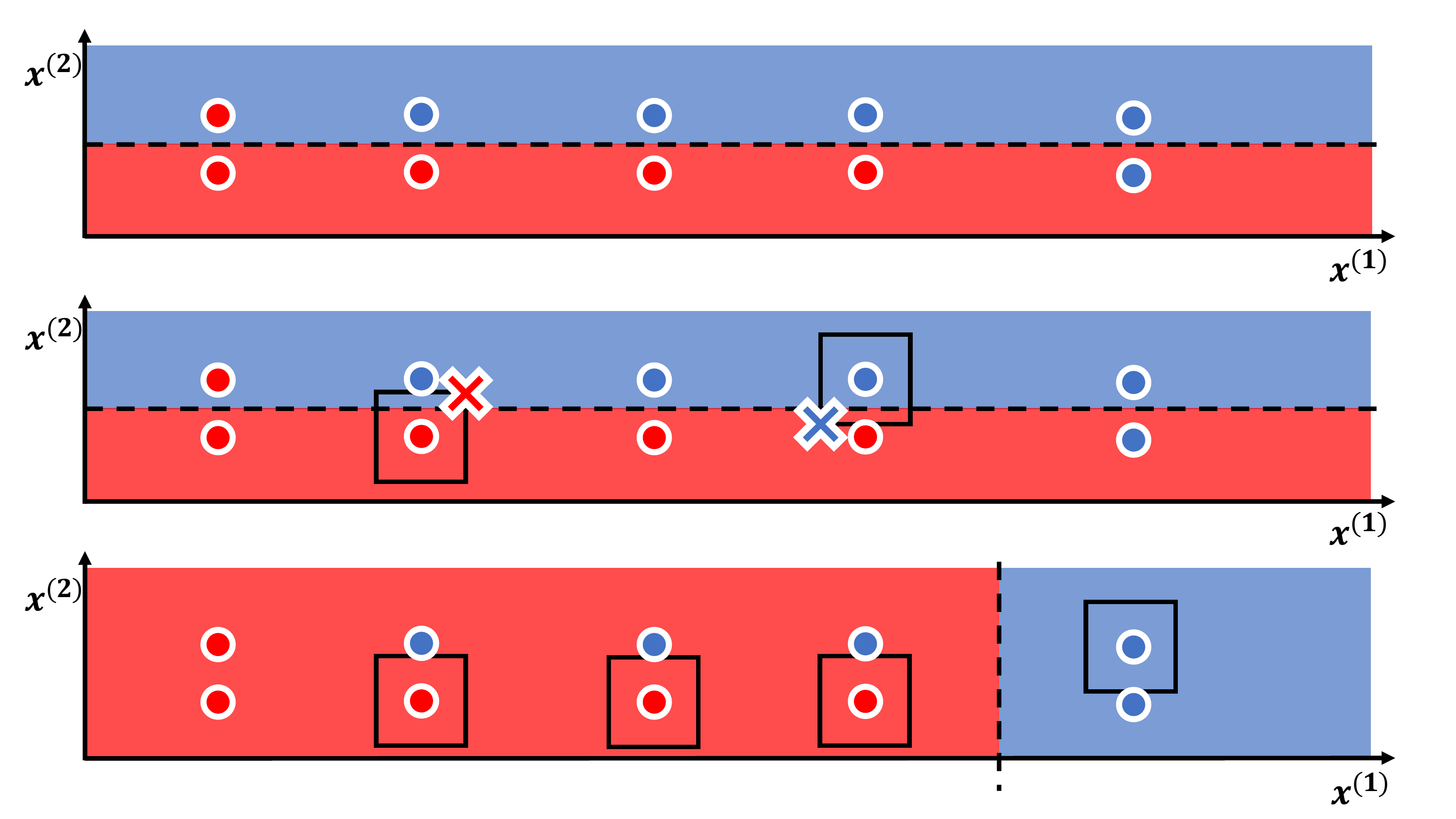}
  \caption{(Best viewed in color) A simple example illustrating how robust splitting works. Upper: A set of 10 points that can be easily separated with a horizontal split on feature $x^{(2)}$. The accuracy of this split is 0.8. Middle: The high accuracy horizontal split cannot separate the $\ell_\infty$ balls around the data points and thus an adversary can perturb any example $\mathbf{x}_i$ within the indicated $\ell_\infty$ ball to mislead the model. The worst case accuracy under adversarial perturbations is 0 if all points are perturbed within the square boxes ($\ell_\infty$ norm bounded noise). Lower: a more robust split would be a split on feature $x^{(1)}$. The accuracy of this split is 0.7 under all possible perturbations within the same size $\ell_\infty$ norm bounded noise (square boxes).}
  \label{fig:toy}
\end{figure}
%%%%%%%%%%%%%%%%%%%%%%%%%%%%%%%%%%%%%%%%%%%%%%%%%%%%%%%%%%%%%%%%%%%%
\subsection{General Robust Decision Tree Framework}
\label{sec:RobustSplitting}
In this section we formally introduce our robust decision tree training framework. For a training set with $N$ examples and $d$ real valued features $\mathcal{D}=\{(\mathbf{x}_i, y_i)\}$ ($1\leq i\leq N$, $y_i\in\mathbb{R}$, $\mathbf{x}_i=[x_i^{(1)},x_i^{(2)},\dots,x_i^{(j)},\dots,x_i^{(d)}]\in\mathbb{R}^d$), we first normalize the feature values to $[0,\ 1]$ such that $\mathbf{x}_i\in[0,1]^d$ (the best feature value for split will also be scaled accordingly, but it is irrelevant to model performance). For a general decision tree based learning model, at a given node, we denote $\mathcal{I}\subseteq\mathcal{D}$ as the set of points at that node. For a split on the $j$-th feature with a threshold $\eta$, %a $\ell_\infty$ ball radius $\epsilon$, 
the sets that will be mentioned in Sections~\ref{sec:RobustSplitting}, ~\ref{sec:infogain} and~\ref{sec:xgboost} are summarized in Table~\ref{tab:notation}. 

%\begin{wraptable}{r}{0.5\textwidth}
\begin{table}[ht]
\centering
\scalebox{0.7}{
\setlength\tabcolsep{1.5pt}
\begin{tabular}{cc}
\Xhline{5\arrayrulewidth}
 Notation&   Definition   \\
 \Xhline{3\arrayrulewidth}
 $\mathcal{I}$& set of examples on the current node      \\
 $\mathcal{I}_0$&   $\mathcal{I}\cap\{(\mathbf{x}_i,y_i)|y_i=0\}$  (for classification) \\
 $\mathcal{I}_1$&   $\mathcal{I}\cap\{(\mathbf{x}_i,y_i)|y_i=1\}$  (for classification) \\
 $\mathcal{I}_L$&   $\mathcal{I}\cap\{(\mathbf{x}_i,y_i)|x^{(j)}<\eta\}$   \\
 $\mathcal{I}_R$&   $\mathcal{I}\cap\{(\mathbf{x}_i,y_i)|x^{(j)}\geq\eta\}$   \\
 $\Delta \mathcal{I}$&    $\mathcal{I}\cap\{(\mathbf{x}_i,y_i)|\eta-\epsilon\leq x^{(j)}\leq\eta+\epsilon$\}   \\
  $\Delta \mathcal{I}_L$&    $\Delta \mathcal{I}\cap\mathcal{I}_L$   \\
  $\Delta \mathcal{I}_R$&    $\Delta \mathcal{I}\cap\mathcal{I}_R$   \\
  $\mathcal{I}_L^o$&    $\mathcal{I}_L\setminus \Delta \mathcal{I}$   \\
  $\mathcal{I}_R^o$&    $\mathcal{I}_R\setminus \Delta \mathcal{I}$   \\
\Xhline{5\arrayrulewidth}
\end{tabular}
}
\caption{Notations of different sets in Section~\ref{sec:robust_tree}. We assume a split is made on the $j$-th feature with a threshold $\eta$, and this feature can be perturbed by $\pm \epsilon$.}
\label{tab:notation}
%\end{wraptable}
\end{table}

In classical tree based learning algorithms (which we refer to as ``natural'' trees in this paper), the quality of a split on a node can be gauged by a \emph{score function} $S(\cdot)$: a function of the splits on left and right child nodes ($\mathcal{I}_L$ and $\mathcal{I}_R$), or equivalently on the chosen feature $j$ to split and a corresponding threshold value~$\eta$. Since $\mathcal{I}_L$ and $\mathcal{I}_R$ are determined by $j$, $\eta$ and $\mathcal{I}$, we abuse the notation and define $S(j,\ \eta,\ \mathcal{I}):=S(\mathcal{I}_L,\ \mathcal{I}_R)$. 

Traditionally, people consider different scores for choosing the ``best'' split, such as information gain used by ID3~\citep{quinlan1986induction} and C4.5~\citep{quinlan1986induction}, or Gini impurity in CART~\citep{breiman2017classification}. Modern software packages~\citep{chen2016xgboost,ke2017lightgbm,dorogush2018catboost} typically find the best split that minimize a loss function directly, allowing decision trees to be used in a large class of problems (i.e., mean square error loss for regression, logistic loss for classification, and ranking loss for ranking problems).
A regular (``natural'') decision tree training process will either exactly or approximately evaluate the score function, for all possible features and split thresholds on the leaf to be split, and select the best $j,\ \eta$ pair:
\begin{equation}
    j^*,\eta^*=\argmax_{j,\ \eta}\ S(\mathcal{I}_L,\ \mathcal{I}_R)=\argmax_{j,\ \eta}\ S(j,\ \eta,\ \mathcal{I}).
\end{equation}
In our setting, we consider the case where features of examples in $\mathcal{I}_L$ and $\mathcal{I}_R$ can be perturbed by an adversary. Since a typical decision tree can only split on a single feature at one time, it is natural to consider adversarial perturbations within an $\ell_\infty$ ball of radius $\epsilon$ around each example $\mathbf{x}_i$:
\[
\mathit{B}_\epsilon^{\infty}(\mathbf{x}_i)\vcentcolon=[x_i^{(1)}-\epsilon,\ x_i^{(1)}+\epsilon]\times\dots\times[x_i^{(d)}-\epsilon,\ x_i^{(d)}+\epsilon].
\]
Such perturbations enable the adversary to minimize the score obtained by our split. So instead of finding a split with highest score, an intuitive approach for robust training is to maximize the minimum score value obtained by all possible perturbations in an $\ell_\infty$ ball with radius $\epsilon$, 
\begin{equation}
\label{eq:maxmin}
    j^*,\ \eta^*=\argmax_{j,\ \eta}\ RS(j,\ \eta,\ \mathcal{I}),
\end{equation}
where $RS(\cdot)$ is a \emph{robust score function} defined as
\begin{equation}
\begin{split}
    &RS(j,\ \eta,\  \mathcal{I})\vcentcolon=\min_{\mathcal{I}'=\{(\mathbf{x}'_i,\ y_i)\}}S(j,\ \eta,\ \mathcal{I}')\\
    &\text{s.t. }\mathbf{x}'_i\in\mathit{B}_\epsilon^{\infty}(\mathbf{x}_i) \text{, for all } \mathbf{x}'_i\in \mathcal{I}'.
\end{split}
\end{equation}
In other words, each $\mathbf{x}_i \in \mathcal{I}$ can be perturbed individually under an $\ell_\infty$ norm bounded perturbation to form a new set of training examples $\mathcal{I}'$. We consider the worst case perturbation, such that the set $\mathcal{I}'$ triggers the worst case score after split with feature $j$ and threshold $\eta$. The training objective~\eqref{eq:maxmin} becomes a max-min optimization problem.

Note that there is an intrinsic consistency between boundaries of the $\ell_\infty$ balls and the decision boundary of a decision tree. For the split on the $j$-th feature, perturbations along features other than $j$ do not affect the split. So we only need to consider perturbations within $\pm \epsilon$ along the $j$-th feature. We define $\Delta\mathcal{I}$ as the \emph{ambiguity set}, containing examples with feature $j$ inside the $[\eta-\epsilon, \eta+\epsilon]$ region (see Table~\ref{tab:notation}).
Only examples in $\Delta\mathcal{I}$ may be perturbed from $\mathcal{I}_L$ to $\mathcal{I}_R$ or from $\mathcal{I}_R$ to $\mathcal{I}_L$ to reduce the score. Perturbing points in $\mathcal{I}\setminus\Delta\mathcal{I}$ will not change the score or the leaves they are assigned to. We denote $\mathcal{I}_L^o$ and $ \mathcal{I}_R^o$ as the set of examples that are certainly on the left and right child leaves under perturbations (see Table~\ref{tab:notation} for definitions). Then we introduce 0-1 variables $s_i=\{0,1\}$ denoting an example in the ambiguity set $\Delta\mathcal{I}$ to be assigned to $\mathcal{I}_L$ and $\mathcal{I}_R$, respectively. Then the $RS$ can be formulated as a 0-1 integer optimization problem with $|\Delta\mathcal{I}|$ variables, which is NP-hard in general. %The $RS$ can be formulated as a 0-1 integer optimization problem with $|\Delta\mathcal{I}|$ variables,
%\begin{equation}
%\begin{split}
%    &RS(j,\ \eta,\  \mathcal{I})=\min_{s_i\in\{0,1\}}S(\mathcal{I}_L',\ \mathcal{I}_R')\\
%    &\text{s.t. }\mathcal{I}_L'=\mathcal{I}_L^o\cup\{(\mathbf{x}_i,\ y_i)\in \Delta\mathcal{I}|s_i=0\},\\&\text{and}\ \mathcal{I}_R'=\mathcal{I}_R^o\cup\{(\mathbf{x}_i,\ y_i)\in \Delta\mathcal{I}|s_i=1\}.
%\end{split}
%\end{equation}
%0-1 integer optimization problems are NP-hard in general. 
Additionally, we need to scan through all $d$ features of all examples and solve $O(|\mathcal{I}|d)$ minimization problems for a single split at a single node. This large number of problems to solve makes this computation intractable. Therefore, we need to find an approximation for the $RS(j,\ \eta,\  \mathcal{I})$. In Sections~\ref{sec:infogain} and~\ref{sec:xgboost}, we present two different approximations and corresponding implementations of our robust decision tree framework, first for classical decision trees with information gain score, and then for modern tree boosting systems which can minimize any loss function.

It is worth mentioning that we normalize features to $[0,\ 1]^d$ for the sake of simplicity in this paper. One can also define $\epsilon_1,\ \epsilon_2,\ \dots,\ \epsilon_d$ for each feature and then the adversary is allowed to perturb $\mathbf{x}_i$ within $[x_i^{(1)}-\epsilon_1,\ x_i^{(1)}+\epsilon_1]\times\dots\times[x_i^{(d)}-\epsilon_d,\ x_i^{(d)}+\epsilon_d]$. In this case, we would not need to normalize the features. Also, $\epsilon$ is a hyper-parameter in our robust model. Models trained with larger $\epsilon$ are expected to be more robust and when $\epsilon=0$, the robust model is the same as a natural model.
%%%%%%%%%%%%%%%%%%%%%%%%%%%%%%%%%%%%%%%%%%%%%%%%%%%%%%%%%%%%
\subsection{Robust Splitting for Decision Trees with Information Gain Score}
\label{sec:infogain}

Here we consider a decision tree for binary classification, $y_i\in\{0,1\}$, with information gain as the metric for node splitting. The information gain score is
\begin{equation*}
S(j,\ \eta,\ \mathcal{I})\vcentcolon=IG(j,\ \eta)=H(y)-H(y|x^{(j)}<\eta),
\end{equation*}
where $H(\cdot)$ and $H(\cdot|\cdot)$ are entropy and conditional entropy on the empirical distribution. For simplicity, we denote $N_0\vcentcolon=|\mathcal{I}_0|$, $N_1\vcentcolon=|\mathcal{I}_1|$, $n_0\vcentcolon=|\mathcal{I}_L\cap \mathcal{I}_0|$ and $n_1\vcentcolon=|\mathcal{I}_L\cap \mathcal{I}_1|$. The following theorem shows adversary's perturbation direction to minimize the information gain.
\begin{theorem}
\label{thm}
If $\frac{n_0}{N_0}<\frac{n_1}{N_1}$ and $\frac{n_0+1}{N_0}\leq\frac{n_1}{N_1}$, perturbing one example in $\Delta \mathcal{I}_R$ with label 0 to $\mathcal{I}_L$ will decrease the information gain.
\end{theorem}
Similarly, if $\frac{n_1}{N_1}<\frac{n_0}{N_0}$ and $\frac{n_1+1}{N_1}\leq\frac{n_0}{N_0}$, perturbing one example in $\Delta \mathcal{I}_R$ with label 1 to $\mathcal{I}_L$ will decrease the information gain. The proof of this theorem will be presented in Section~\ref{sec:proof} in the appendix. Note that we also have a similar conclusion for Gini impurity score, which will be shown in Section~\ref{sec:gini} in the appendix.  Therefore, to decrease the information gain score, the adversary needs to perturb examples in $\Delta\mathcal{I}$ such that $\frac{n_0}{N_0}$ and $\frac{n_1}{N_1}$ are close to each other (the ideal case $\frac{n_0}{N_0} = \frac{n_1}{N_1}$ may not be achieved because $n_0$, $n_1$, $N_0$ and $N_1$ are integers). The robust split finding algorithm is shown in Algorithm~\ref{alg:robust_infogain_split}. In this algorithm we find a perturbation that minimizes $\left|\frac{n_0}{N_0}-\frac{n_1}{N_1}\right|$ as an approximation and upper bound to the optimal solution. Algorithm~\ref{alg:split_min} in Section~\ref{sec:proof} in the appendix shows an $O(|\mathcal{I}|)$ procedure to find such perturbation to approximately minimize the information gain. Since the algorithm scans through $\{x^{(j)}_1, \dots, x^{(j)}_d\}$ in the sorted order, the sets $\Delta \mathcal{I}$, $\mathcal{I}^o_L, \mathcal{I}^o_R$ can be maintained in amortized $O(1)$ time in the inner loop.  Therefore, the computational complexity of the robust training algorithm is $O(d|\mathcal{I}|^2)$ per split. 

Although it is possible to extend our conclusion to other traditional scores of classification trees, we will focus on the modern scenario where we use a regression tree to fit any loss function in Section~\ref{sec:xgboost}.
%which takes $O(N\log N)$ time.
%For any split, we can count $m$ and $n$. If $n/N<m/M$, we keep perturbing the $j$-th feature of points with label 0 in $(\eta,\eta+\epsilon)$ to $(-\infty, \eta)$. We also keep perturbing the $j$-th feature of points with label 1 in $(\eta-\epsilon,\eta)$ to $(\eta,+\infty)$. Stop this procedure until no more available points or $n/N\geq m/M$. Then the information gain at this time is $RIG$. Calculating $RIG$ takes at most $O(M+N)$ time. Then in decision tree training process, we simply maximize $RIG$ instead of $IG$.

%In decision trees, $\mathcal{I}_L$ and $\mathcal{I}_R$ are determined by threshold $\eta$ on a feature $k$. Figure~\ref{fig:RobustSplit} is a toy example. Then the adversary can perturb points within $(\eta-\epsilon,\eta+\epsilon)$ to make $\frac{n}{N}$ and $\frac{m}{M}$ closer and to decrease the information gain metric. For example, in Figure~\ref{fig:RobustSplit}, the adversary may move the 4 points within $(\eta-\epsilon,\eta+\epsilon)$ to make the information gain from 0.0817 to 0, as shown in Figure~\ref{fig:AdvSplit}, which means that this particular split on this attribute is not robust.
\begin{algorithm}[tb]
\caption{Robust Split with Information Gain}
\label{alg:robust_infogain_split}
\begin{algorithmic}
\STATE {\bfseries Input:} Training set $\{(x_i,\ y_i)\}|_{i=1}^N$, $x_i\in [0,1]^d$, $y_i\in\{0,1\}$.
\STATE {\bfseries Input:} The instance set of the current node $I$. 
\STATE {\bfseries Input:} $\epsilon$, the radius of the $\ell_\infty$ ball.
\STATE {\bfseries Output:} Optimal split of the current node.
    \STATE$\mathcal{I}_0\leftarrow \{(x_i,\ y_i)|y_i=0\},\mathcal{I}_1\leftarrow \{(x_i,\ y_i)|y_i=1\}$;
    \STATE$N_0\leftarrow |\mathcal{I}\cap \mathcal{I}_0|,\ N_1\leftarrow|\mathcal{I}\cap \mathcal{I}_1|$;
    \FOR{$j\leftarrow 1$ {\bfseries to} $d$}
        \FOR{$m$ in sorted($\mathcal{I}$, ascending order by $x_m^j$)}
            \STATE$\eta\leftarrow \frac{1}{2}(x_m^j+x_{m+1}^j)$, $\Delta\mathcal{I}\leftarrow \mathcal{I}\cap\{(x_i,\ y_i)|\eta-\epsilon\leq x^{(j)}\leq\eta+\epsilon\}$;
            \STATE$\mathcal{I}_L^o\leftarrow\{(x_i,\ y_i)|x^{(j)}<\eta-\epsilon\}, \mathcal{I}_R^o\leftarrow\{(x_i,\ y_i)|x^{(j)}>\eta+\epsilon\}$;
            \STATE$n_0^o\leftarrow|\mathcal{I}_L^o\cap \mathcal{I}_0|$, $n_1^o\leftarrow|\mathcal{I}_L^o\cap \mathcal{I}_1|$;
            \STATE Find $\Delta n_0^*$, $\Delta n_1^*$ to minimize $|\frac{\Delta n_0^*+n_0^o}{N_0}-\frac{\Delta n_1^*+n_1^o}{N_1}|$ using Algorithm~\ref{alg:split_min} in Section~\ref{sec:proof} in the appendix;
            \STATE From $\Delta\mathcal{I}$, add $\Delta n_0^*$ points with $y=0$ and $\Delta n_1^*$ points with $y=1$ to $\mathcal{I}_L^o$ and obtain $\mathcal{I}_L$;
            \STATE Add remaining points in $\Delta \mathcal{I}$ to $\mathcal{I}_R^o$ and obtain $\mathcal{I}_R$;
            \STATE$RS(j,\ \eta)\leftarrow IG(\mathcal{I}_L,\mathcal{I}_R)$;
        \ENDFOR
    \ENDFOR
    \STATE$j^*,\eta^*\leftarrow\argmax_{j,\ \eta} RS(j,\ \eta)$;
    \STATE Split on feature $j^*$ with a threshold $\eta^*$;
    \end{algorithmic}
\end{algorithm}

%%%%%%%%%%%%%%%%%%%%%%%%%%%%%%%%%%%%%%%%%%%%%%%%%%%%%%%%%%%%%%%%%%%%

\subsection{Robust Splitting for GBDT models}
\label{sec:xgboost}
% XGBoost~\citep{chen2016xgboost} is a state-of-the-art large scale tree boosting system which has been integrated into industrial production learning pipelines of many major companies.

\begin{algorithm}
\caption{Robust Split for Boosted Tree}
\label{alg:robust_xgboost_split}
\begin{algorithmic}
\STATE {\bfseries Input:} training set $\{(x_i,\ y_i)\}|_{i=1}^N$, $x_i\in [0,1]^d$, $y_i\in\mathbb{R}$.
\STATE {\bfseries Input:} The instance set of the current node $I$.
\STATE {\bfseries Input:} $\epsilon$, the radius of the $\ell_\infty$ ball.
\STATE {\bfseries Output:} Optimal split of the current node.
    \FOR{$j\leftarrow 1$ {\bfseries to} $d$}
        \FOR{$m$ in sorted($\mathcal{I}$, ascending order by $x_m^j$)}
            \STATE$\eta\leftarrow \frac{1}{2}(x_m^j+x_{m+1}^j)$; \STATE$\mathcal{I}_L^o\leftarrow\{(x_i,\ y_i)|x^{(j)}<\eta-\epsilon\}, \Delta \mathcal{I}_L\leftarrow \mathcal{I}\cap\{(x_i,\ y_i)|\eta-\epsilon\leq x^{(j)}<\eta\}$;
            \STATE$\mathcal{I}_R^o\leftarrow\{(x_i,\ y_i)|x^{(j)}>\eta+\epsilon\}$,  $\Delta \mathcal{I}_R\leftarrow \mathcal{I}\cap\{(x_i,\ y_i)|\eta\leq x^{(j)}\leq\eta+\epsilon\}$;
            \STATE $S_1=S(\mathcal{I}_L,\mathcal{I}_R)$,\  $S_2=S(\mathcal{I}_L^o,\ \mathcal{I}_R^o\cup\Delta \mathcal{I})$,\ $S_3=S(\mathcal{I}_L^o\cup\Delta \mathcal{I},\mathcal{I}_R^o)$, $S_4=S(\mathcal{I}_L^o\cup\Delta \mathcal{I}_R, \mathcal{I}_R^o\cup\Delta \mathcal{I}_L)$;
            \STATE $RS(j,\ \eta)\leftarrow \min\{S_1,S_2,S_3,S_4\};$
        \ENDFOR
    \ENDFOR
    \STATE$j^*,\eta^*\leftarrow\argmax_{j,\ \eta} RS(j,\ \eta)$;
    \STATE Split on feature $j^*$ with a threshold $\eta^*$;
  \end{algorithmic}
\end{algorithm}

We now introduce the regression tree training process used in many modern tree boosting packages including XGBoost~\citep{chen2016xgboost}, LightGBM~\citep{ke2017lightgbm} and CatBoost~\citep{dorogush2018catboost}.
Specifically, we focus on the formulation of gradient boosted decision tree (GBDT), which is one of the most successful ensemble models and has been widely used in industry. GBDT is an additive tree ensemble model $\phi(\cdot)$ combining outputs of $K$ trees 
\[
\hat{y}_i=\phi_K(\mathbf{x}_i)=\sum_{k=1}^K f_k(\mathbf{x}_i)
\]
where each $f_k$ is a decision tree and $\hat{y}_i$ is the final output for $\mathbf{x}_i$. 
Here we only focus on regression trees where $\hat{y}_i \in \mathbb{R}$.
Note that even for a classification problem, the modern treatment in GBDT is to consider the data with logistic loss, and use a regression tree to minimize this loss.

During GBDT training, the trees $f_k$ are generated in an additive manner: when we consider the tree $f_K$, all previous trees $f_k, k \in \{1, \cdots, K-1\}$ are kept unchanged.
For a general convex loss function $l$ (such as MSE or logistic loss), we desire to minimize the following objective
\begin{equation*}
\begin{split}
    &\mathcal{L}(\phi, \mathcal{D})=\sum_{i=1}^N l(y_i,\hat{y}_i)+\sum_{k=1}^K\Omega(f_k) \\&=\sum_{i=1}^N l\left(y_i, \phi_{K-1}(\mathbf{x}_i) + f_K(\mathbf{x}_i)\right)+\sum_{k=1}^{K-1}\Omega(f_k) + \Omega(f_K)
\end{split}
\label{eq:xgboost_loss}
\end{equation*}
where $\Omega(f)$ is a regularization term to penalize complex trees; for example, in XGBoost, $\Omega(f)=\gamma T+\frac{1}{2}\lambda\|\mathbf{\omega}\|^2$, where $T$ is the number of leaves, $\mathbf{\omega}$ is a vector of all leaf predictions and $\lambda,\ \gamma\geq 0$ are regularization constants.
Importantly, when we consider $f_K$, $\phi_{K-1}$ is a constant. The impact of $f_K(\mathbf{x}_i)$ on $l(y_i, \hat{y}_i)$ can be approximated using a second order Taylor expansion:
\begin{equation*}
\begin{split}
  &l(y_i, \phi_{K}(\mathbf{x}_i)) \approx \hat{l}(y_i, \phi_{K}(\mathbf{x}_i)) \\&:= l(y_i, \phi_{K-1}(\mathbf{x}_i)) + g_i f_K(\mathbf{x}_i) + \frac{1}{2} h_i (f_K(\mathbf{x}_i))^2
\end{split}
\end{equation*}
where $g_i=\frac{\partial l(y_i, \phi_{K}(\mathbf{x}_i))}{\partial f_K(\mathbf{x}_i)}$ and
$h_i=\frac{\partial^2 l(y_i, \phi_{K}(\mathbf{x}_i))}{\partial f_K^2(\mathbf{x}_i)}$ are the first and second order derivatives on the loss function with respect to the prediction of decision tree $f_K$ on point $\mathbf{x}_i$. Conceptually, ignoring the regularization terms, the score function can be given as:
\begin{equation*}
\begin{split}
  &S(\mathcal{I}_L, \mathcal{I}_R) = \sum_{i\in \mathcal{I}_L} \hat{l}(y_i, \phi_{K}(\mathbf{x}_i)) \rvert_{\phi_{K}(\mathbf{x}_i)=\omega_L} \\& +\! \sum_{i\in \mathcal{I}_R}\! \hat{l}(y_i, \phi_{K}(\mathbf{x}_i)) \rvert_{\phi_{K}(\mathbf{x}_i)=\omega_R} \!-\! \sum_{i\in \mathcal{I}} \!\hat{l}(y_i, \phi_{K}(\mathbf{x}_i)) \rvert_{\phi_{K}(\mathbf{x}_i)=\omega_P}
\end{split}
\end{equation*}
where $\omega_L$, $\omega_R$ and $\omega_P$ are the prediction values of the left, right and parent nodes. The score represents the improvements on reducing the loss function $\mathcal{L}$ for all data examples in $\mathcal{I}$.
The exact form of score used in XGBoost with regularization terms is given in~\cite{chen2016xgboost}:
\begin{equation*}
\begin{split}
    &S(j,\ \eta,\ \mathcal{I})=S(\mathcal{I}_L,\ \mathcal{I}_R) \\&\!:=\!
    \frac{1}{2}\bigg[\frac{(\sum_{i\in \mathcal{I}_L}g_i)^2}{\sum_{i\in \mathcal{I}_L}h_i+\lambda}\!+\!\frac{(\sum_{i\in \mathcal{I}_R}g_i)^2}{\sum_{i\in \mathcal{I}_R}h_i+\lambda}\!-\!\frac{(\sum_{i\in \mathcal{I}}g_i)^2}{\sum_{i\in \mathcal{I}}h_i+\lambda}\bigg]-\gamma,
    \label{eq:xgboost_split_score}
\end{split}
\end{equation*}
where $\gamma$ is a regularization constant. Again, to minimize the score by perturbing points in $\Delta\mathcal{I}$, the adversary needs to solve an intractable 0-1 integer optimization at each possible splitting position. Since GBDT is often deployed in large scale data mining tasks with a large amount of training data to scan through at each node, and we need to solve $RS$ $O(|\mathcal{I}|d)$ times, we cannot afford any expensive computation. For efficiency, our robust splitting procedure for boosted decision trees, as detailed in Algorithm~\ref{alg:robust_xgboost_split}, approximates the minimization by considering only four representative cases:
(1) no perturbations: $S_1=S(\mathcal{I}_L,\ \mathcal{I}_R)$;
(2) perturb all points in $\Delta\mathcal{I}$ to the right: $S_2=S(\mathcal{I}_L^o, \mathcal{I}_R^o\cup\Delta \mathcal{I})$;
(3) perturb all points in $\Delta\mathcal{I}$ to the left: $S_3=S(\mathcal{I}_L^o\cup\Delta \mathcal{I},\mathcal{I}_R^o)$;
(4) swap the points in $\Delta\mathcal{I}$: $S_4=S(\mathcal{I}_L^o\cup\Delta \mathcal{I}_R, \mathcal{I}_R^o\cup\Delta \mathcal{I}_L)$.
We take the minimum among the four representative cases as an approximation of the $RS$: 
\begin{equation}
RS(j,\ \eta,\ \mathcal{I})\approx \min\{S_1,\ S_2,\ S_3,\ S_4\}.
\end{equation}Though this method only takes $O(1)$ time to give a rough approximation of the $RS$ at each possible split position, it is effective empirically as demonstrated next in Section~\ref{sec:exp}.

\section{Experiments}
\label{sec:exp}
Our code is at \url{https://github.com/chenhongge/RobustTrees}.\vspace{-3mm}
%We present experimental results to study the effectiveness of our proposed robust tree methods. We first show robust information gain based decision trees on small datasets in Section~\ref{sec:info_gain_exp}, and then demonstrate robust XGBoost models on large and challenging datasets in Section~\ref{sec:xgboost_exp}. 
%%%%%%%%%%%%%%%%%%%%%%%%%%%%%%%%%%%%%%%%%%%%%%%%%%%%%%%%%
\subsection{Robust Information Gain Decision Trees}
\label{sec:info_gain_exp}
We present results on three small datasets with robust information gain based decision trees using Algorithm~\ref{alg:robust_infogain_split}. We focus on untargeted adversarial attacks. For each dataset we test on 100 examples (or the whole test set), and we only attack correctly classified images. Attacks proceed until the attack success rate is 100\%; the differences in robustness are reflected in the distortion of the adversarial examples required to achieve a successful attack. In Table~\ref{tab:vanilla_robust_acc_table}, we present the average $\ell_\infty$ distortion of the adversarial examples of both classical natural decision trees and our robust decision trees trained on different datasets. We use Papernot's attack as well as $\ell_\infty$ versions of Cheng's and Kantchelian's attacks. The $\ell_1$ and $\ell_2$ distortion found by Kantchelian's $\ell_1$ and $\ell_2$ attacks are presented in Table~\ref{tab:vanilla_robust_acc_table_l1l2} in the appendix. The adversarial examples found by Cheng's, Papernot's and Kantchelian's attacks have much larger $\ell_\infty$ norm for our robust trees compared to those for the natural trees, demonstrating that our robust training algorithm improves the decision tree robustness substantially. In some cases our robust decision trees also have higher test accuracy than the natural trees. This may be due to the fact that the robust score tends to encourage the tree to split at thresholds where fewer examples are in the ambiguity set, and thus the split is also robust against random noise in the training set. Another possible reason is the implicit regularization in the robust splitting. The robust score is always lower than the regular score and thus our splitting is more conservative. Also, from results in Table~\ref{tab:vanilla_robust_acc_table} we see that most of the adversarial examples found by Papernot's attack have larger $\ell_\infty$ norm than those found by Cheng's $\ell_\infty$ attack. This suggests that the straight-forward greedy search attack is not as good as a sophisticated general attack for attacking decision trees. Cheng's attack is able to achieve similar $\ell_\infty$ distortion as Kantchelian's attack, without solving expensive MILPs. While not scalable to large datasets, Kantchelian's attack can find the \emph{minimum} adversarial examples, reflecting the true robustness of a tree-based model.

\begin{table*}[htb]
\begin{center}
\scalebox{0.6}{
\setlength\tabcolsep{1.5pt}
\begin{tabular}{c|c|c|c|c|c|c|c|c|c|c|c|c|c|c|c}

\Xhline{5\arrayrulewidth}
    
    \multirow{2}*{Dataset}&training&test&\# of&\# of &\multirow{2}*{robust $\epsilon$}& \multicolumn{2}{c|}{depth}&\multicolumn{2}{c|}{test acc.}&\multicolumn{2}{c|}{\thead{avg. $\ell_\infty$ dist.\\by Cheng's $\ell_\infty$ attack}}&\multicolumn{2}{c|}{\thead{avg. $\ell_\infty$ dist.\\by Papernot's attack}}&\multicolumn{2}{c}{\thead{avg. $\ell_\infty$ dist.\\by Kantchelian's $\ell_\infty$ attack}}\\
    &set size&set size&features&classes&&robust&natural&robust&natural&robust&natural&robust&natural&robust&natural\\\Xhline{3\arrayrulewidth}
    
    breast-cancer &546&137&10&2&0.3&5&5&.948&.942&\textbf{.531}&.189&\textbf{.501}&.368&\textbf{.463}&.173\\
    
    diabetes&614&154&8&2&0.2&5&5&.688&.747&\textbf{.206}&.065&\textbf{.397}&.206&\textbf{.203}&.060\\
    
    ionosphere&281&70&34&2&0.2&4&4&.986&.929&\textbf{.388}&.109&\textbf{.408}&.113&\textbf{.358}&.096\\
    
    \Xhline{5\arrayrulewidth}
    
\end{tabular}
}
\caption{Test accuracy and robustness of information gain based single decision tree model. The robustness is evaluated by the average $\ell_\infty$ distortion of adversarial examples found by Cheng's, Papernot's and Kantchelian's attacks. Average $\ell_\infty$ distortion of robust decision tree models found by three attack methods are consistently larger than that of the naturally trained ones.}
\label{tab:vanilla_robust_acc_table}
\end{center}
\vspace{-4mm}
\end{table*}

%%%%%%%%%%%%%%%%%%%%%%%%%%%%%%%%%%%%%%%%%%%%%%%%%%%%%%%%%
\subsection{Robust GBDT Models}
\label{sec:xgboost_exp}
In this subsection, we evaluate our algorithm in the tree boosting setting, where multiple robust decision trees are created in an ensemble to improve model accuracy. We implement Algorithm~\ref{alg:robust_xgboost_split} by slightly modifying the node splitting procedure in XGBoost. Our modification is only relevant to computing the scores for selecting the best split, and is compatible with other existing features of XGBoost. We also use XGBoost to train natural (undefended) GBDT models. Again, we focus on untargeted adversarial attacks. We consider nine real world large or medium sized datasets and two small datasets~\cite{CC01a}, spanning a variety of data types (including both tabular and image data). For small datasets we use 100 examples and for large or medium sized datasets, we use 5000 examples for robustness evaluation, except for MNIST 2 vs. 6, where we use 100 examples. MNIST 2 vs. 6 is a subset of MNIST to only distinguish between 2 and 6. This is the dataset tested in~\citet{kantchelian2016evasion}. We use the same number of trees, depth and step size shrinkage as in~\citet{kantchelian2016evasion} to train our robust and natural models. Same as~\citet{kantchelian2016evasion}, we only test 100 examples for MNIST 2 vs. 6 since the model is relatively large. In Table~\ref{tab:small_robust_acc_table}, we present the average $\ell_\infty$ distortion of  adversarial examples found by Cheng's $\ell_\infty$ attack for both natural GBDT and robust GBDT models trained on those datasets. For small and medium binary classification models, we also present  results of Kantchelian's $\ell_\infty$ attack, which finds the \emph{minimum} adversarial example in $\ell_\infty$ norm. The $\ell_1$ and $\ell_2$ distortion found by Kantchelian's $\ell_1$ and $\ell_2$ attacks are presented in Table~\ref{tab:small_robust_acc_table_l1l2} in the appendix. Kantchelian's attack can only handle binary classification problems and small scale models due to its time-consuming MILP formulation. Papernot's attack is inapplicable here because it is for attacking a single tree only. The natural and robust models have the same number of trees for comparison. We only attack correctly classified images and all examples are successfully attacked. We see that our robust GBDT models consistently outperform the natural GBDT models in terms of $\ell_\infty$ robustness. %Our method also improves $\ell_1$ and $\ell_2$ robustness (results deferred to appendix).

For some datasets, we need to increase tree depth in robust GBDT models in order to obtain accuracy comparable to the natural GBDT models. The requirement of larger model capacity is common in the adversarial training literature: in the state-of-the-art defense for DNNs, \citet{madry2017towards} argues that increasing the model capacity is essential for adversarial training to obtain good accuracy. 

\begin{table*}[htb]
\begin{center}
\scalebox{0.62}{
\setlength\tabcolsep{1.5pt}
\begin{tabular}{c|c|c|c|c|c|c|c|c|c|c|c|c|c|c|c|c}

\Xhline{5\arrayrulewidth}
    
    \multirow{2}*{Dataset}&training&test&\# of&\# of &\# of &robust& \multicolumn{2}{c|}{depth}&\multicolumn{2}{c|}{test acc.}&\multicolumn{2}{c|}{\thead{avg. $\ell_\infty$ dist.\\by Cheng's $\ell_\infty$ attack}}&dist.&\multicolumn{2}{c|}{\thead{avg. $\ell_\infty$ dist.\\by Kantchelian's $\ell_\infty$ attack}}&dist.\\
    &set size&set size&features&classes&trees& $\epsilon$&robust&natural&robust&natural&robust&natural&improv.&robust&natural&improv.\\\Xhline{3\arrayrulewidth}
    
    breast-cancer &546&137&10&2&4&0.3&8&6&.978&.964&\textbf{.411}&.215&\textbf{1.91X}&\textbf{.406}&.201&\textbf{2.02X}\\
    covtype&400,000&181,000&54&7&80&0.2&8&8&.847&.877&\textbf{.081}&.061&\textbf{1.31X}&not binary&not binary&---\\
    cod-rna&59,535&271,617&8&2&80&0.2&5&4&.880&.965&\textbf{.062}&.053&\textbf{1.16X}&\textbf{.054}&.034&\textbf{1.59X}\\
    diabetes&614&154&8&2&20&0.2&5&5&.786&.773&\textbf{.139}&.060&\textbf{2.32X}&\textbf{.114}&.047&\textbf{2.42X}\\
    Fashion-MNIST&60,000&10,000&784&10&200&0.1&8&8&.903&.903&\textbf{.156}&.049&\textbf{3.18X}&not binary&not binary&---\\
    HIGGS&10,500,000&500,000&28&2&300&0.05&8&8&.709&.760&\textbf{.022}&.014&\textbf{1.57X}&time out& time out&---\\
    ijcnn1&49,990&91,701&22&2&60&0.1&8&8&.959&.980&\textbf{.054}&.047&\textbf{1.15X}&\textbf{.037}&.031&\textbf{1.19X}\\
    MNIST&60,000&10,000&784&10&200&0.3&8&8&.980&.980&\textbf{.373}&.072&\textbf{5.18X}&not binary&not binary&---\\
    Sensorless&48,509&10,000&48&11&30&0.05&6&6&.987&.997&\textbf{.035}&.023&\textbf{1.52X}&not binary&not binary&---\\
    webspam&300,000&50,000&254&2&100&0.05&8&8&.983&.992&\textbf{.049}&.024&\textbf{2.04X}&time out&time out&---\\%\textbf{.015}&&\\
    MNIST 2 vs. 6&11,876&1,990&784&2&1000&0.3&6&4&.997&.998&\textbf{.406}&.168&\textbf{2.42X}&\textbf{.315}&.064&\textbf{4.92X}\\%\textbf{.015}&&\\
    \Xhline{5\arrayrulewidth}

\end{tabular}
}
\caption{The test accuracy and robustness of GBDT models. Average $\ell_\infty$ distortion of our robust GBDT models are consistently larger than those of the naturally trained models. The robustness is evaluated by the average $\ell_\infty$ distortion of adversarial examples found by Cheng's and Kantchelian's attacks. Only small or medium sized binary classification models can be evaluated by Kantchelian's attack, but it finds the minimum adversarial example with smallest possible distortion.}
\label{tab:small_robust_acc_table}
\end{center}
\vspace{-4mm}
\end{table*}

Figure~\ref{fig:robustness_vs_acc_xgboost_mnist} and Figure~\ref{fig:robustness_vs_acc_xgboost_fashion} in the appendix show the distortion and accuracy of MNIST and Fashion-MNIST models with different number of trees. The adversarial examples are found by Cheng's $\ell_\infty$ attack. Models with $k$ trees are the first $k$ trees during a single boosting run of $K$ ($K \geq k$) trees. The $\ell_\infty$ distortion of robust models are consistently much larger than those of the natural models. 
% Here we only plot $\ell_\infty$ distortion. Results for $\ell_1$ and $\ell_2$ are presented in the appendix. Although our method is based on $\ell_\infty$ balls, the $\ell_1$ and $\ell_2$ robustness are also significantly improved. 
For MNIST dataset, our robust GBDT model loses accuracy slightly when the model has only 20 trees. This loss is gradually compensated as more trees are added to the model; regardless of the number of trees in the model, the robustness improvement is consistently observed, as our robust training is embedded in each tree's building process and we create robust trees beginning from the very first step of boosting. Adversarial training in~\citet{kantchelian2016evasion}, in contrast, adds adversarial examples with respect to the current model at each boosting round so adversarial examples produced in the later stages of boosting are only learned by part of the model. The non-robust trees in the first few rounds of boosting still exist in the final model and they may be the weakness of the ensemble. Similar problems are not present in DNN adversarial training since the whole model is exposed to new adversarial examples throughout the training process. This may explain why adversarial training in~\citet{kantchelian2016evasion} failed to improve $\ell_1$, $\ell_2$, or $\ell_\infty$ robustness on the MNIST 2 vs. 6 model, while our method achieves significant robustness improvement with the same training parameters and evaluation metrics, as shown in Tables~\ref{tab:small_robust_acc_table} and~\ref{tab:small_robust_acc_table_l1l2}. Additionally, we also evaluate the robustness of natural and robust models with different number of trees on a variety of datasets using Cheng's $\ell_\infty$ attack, presented in Table~\ref{tab:big_robust_acc_table} in the appendix.

\begin{figure}[htb]\centering
\includegraphics[width=0.9\linewidth]{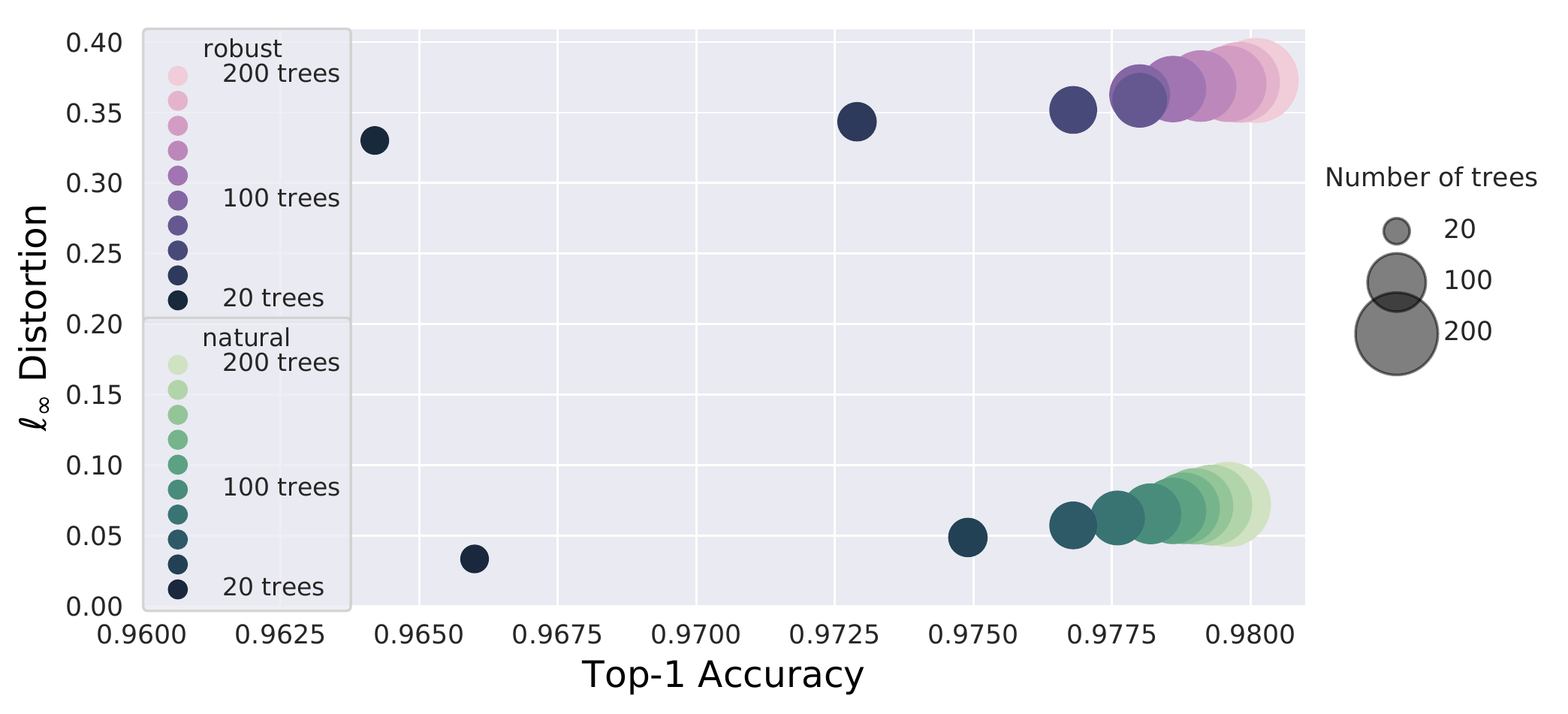}
  
  \caption{(Best viewed in color) $\ell_\infty$ distortion vs. classification accuracy of GBDT models on MNIST dataset with different numbers of trees (circle size). The adversarial examples are found by Cheng's $\ell_\infty$ attack. The robust training parameter $\epsilon=0.3$ for MNIST. With robust training (purple) the distortion needed to fool a model increases dramatically with less than $1\%$ accuracy loss.} %Similarly significant distortion increase also exists in $\ell_1$ and $\ell_2$ norms. We display the $\ell_1$ and $\ell_2$ distortion results in Figures~\ref{fig:robustness_vs_acc_xgboost_mnist_l1_l2} and~\ref{fig:robustness_vs_acc_xgboost_fashion_l1_l2} in the appendix.} %\textcolor{red}{Note that I cropped the left figure to save space and improve formatting, both figures just share "Number of trees" legend on the right. In these figures, can you make the small circle smaller and large circle larger? Their sizes look very similar... Need to increase the contrast. Also, can you make the text on the legend inside the figure larger. You can use a color gradient filled long bar from light to dark colors, and only mark the two ends with numbers. Near the middle of the bar, say ``normal'' or ``robust''. Also, if we start the y-axis from 0, does the figure look okay? It will show the improvements better.}}
  \label{fig:robustness_vs_acc_xgboost_mnist}
\end{figure}

We also test our framework on random forest models and the results are shown in Section~\ref{sec:rf} in the appendix.

\vspace{-0.5cm}\section{Conclusion}
\label{sec:conclusion}
In this paper, we study the robustness of tree-based machine learning models under adversarial attacks. Our experiments show that just as in DNNs, tree-based models are also vulnerable to adversarial attacks. To address this issue, we propose a novel robust decision tree training framework. We make necessary approximations to ensure scalability and implement our framework in both classical decision tree and tree boosting settings. Extensive experiments on a variety of datasets show that our method substantially improves model robustness. Our framework can be extended to other tree-based models such as Gini impurity based classification trees, random forest, and CART. 
% We believe that this paper provides potential means for theoretical study and improvement of the robustness of tree-based models.

\clearpage
\section*{Acknowledgements}
The authors thank Aleksander M\k{a}dry for fruitful discussions. The authors also acknowledge the support of NSF via IIS-1719097, Intel, Google Cloud, Nvidia, SenseTime and IBM.
%\clearpage

\bibliography{ICML-2019-TreeAdvAttack}
\bibliographystyle{icml2019}
\newpage
\clearpage
\appendix
%\section*{Appendices}
%\addcontentsline{toc}{section}{Appendices}
%\renewcommand{\thesection}{\Alph{section}}
\section{Proof of Theorem \ref{thm}}
\label{sec:proof}
Here we prove Theorem \ref{thm} for information gain score.
\begin{proof}
$H(y)$ and $H(y|x^{(j)}<\eta)$ are defined as $$H(y)=-\frac{|\mathcal{I}_0|}{|\mathcal{I}|}\log(\frac{|\mathcal{I}_0|}{|\mathcal{I}|})-\frac{|\mathcal{I}_1|}{|\mathcal{I}|}\log(\frac{|\mathcal{I}_1|}{|\mathcal{I}|}),$$
and %\textcolor{red}{double check this, second line, third $I_0$ should be $I_1$?}
\begin{equation*}\begin{split}
&H(y|x^{(j)}<\eta)=\\
&-\frac{|\mathcal{I}_L|}{|\mathcal{I}|}\bigg[\frac{|\mathcal{I}_L\cap \mathcal{I}_0|}{|\mathcal{I}_L|}\log(\frac{|\mathcal{I}_L\cap \mathcal{I}_0|}{|\mathcal{I}_L|})+\frac{|\mathcal{I}_L\cap \mathcal{I}_1|}{|\mathcal{I}_L|}\log(\frac{|\mathcal{I}_L\cap \mathcal{I}_1|}{|\mathcal{I}_L|})\bigg]\\
&-\frac{|\mathcal{I}_R|}{|\mathcal{I}|}\bigg[\frac{|\mathcal{I}_R\cap \mathcal{I}_0|}{|\mathcal{I}_R|}\log(\frac{|\mathcal{I}_R\cap \mathcal{I}_0|}{|\mathcal{I}_R|})+\frac{|\mathcal{I}_R\cap \mathcal{I}_1|}{|\mathcal{I}_R|}\log(\frac{|\mathcal{I}_R\cap \mathcal{I}_1|}{|\mathcal{I}_R|})\bigg].\end{split}\end{equation*} For simplicity, we denote $N_0\vcentcolon=|\mathcal{I}_0|$, $N_1\vcentcolon=|\mathcal{I}_1|$, $n_0\vcentcolon=|\mathcal{I}_L\cap \mathcal{I}_0|$ and $n_1\vcentcolon=|\mathcal{I}_L\cap \mathcal{I}_1|$.
The information gain of this split can be written as a function of $n_0$ and $n_1$:
\begin{equation}
\begin{split}
IG&=C_1[n_0\log(\frac{n_0}{N_0(n_1+n_0)})+n_1\log(\frac{n_1}{N_1(n_1+n_0)})\\&+(N_0-n_0)\log(\frac{N_0-n_0}{N_0(N_1+N_0-n_1-n_0)})\\&+(N_1-n_1)\log(\frac{N_1-n_1}{N_1(N_1+N_0-n_1-n_0)})]+C_2,%\\
%&=C_1\cdot[n_0\log(n_0)+n_1\log(n_1)-(n_1+n_0)\log(n_1+n_0)\\
%&+(N_0-n_0)\log(N_0-n_0)+(N_1-n_1)\log(N_1-n_1)\\
%&-(N_1+N_0-n_1-n_0)\log(N_1+N_0-n_1-n_0)]+C_3,
\end{split}
\label{eq:IG}
\end{equation}
where $C_1>0$ and $C_2$ are constants with respect to $n_0$. Taking $n_0$ as a continuous variable, we have
\begin{equation}
\frac{\partial IG}{\partial n_0}%=C\cdot\log(\frac{n_0(N_1+N_0-n_1-n_0)}{(N_0-n_0)(n_1+n_0)})
=C_1\cdot\log(1+\frac{n_0N_1-N_0n_1}{(N_0-n_0)(n_1+n_0)})
\end{equation}
When $\frac{\partial IG}{\partial n_0}<0$, perturbing one example in $\Delta \mathcal{I}_R$ with label 0 to $\mathcal{I}_L$ will increase $n_0$ and decrease the information gain. It is easy to see that $\frac{\partial IG}{\partial n_0}<0$ if and only if $\frac{n_0}{N_0}<\frac{n_1}{N_1}.$
This indicates that when $\frac{n_0}{N_0}<\frac{n_1}{N_1}$ and $\frac{n_0+1}{N_0}\leq\frac{n_1}{N_1}$, perturbing one example with label 0 to $\mathcal{I}_L$ will always decrease the information gain. 
\end{proof}Similarly, if $\frac{n_1}{N_1}<\frac{n_0}{N_0}$ and $\frac{n_1+1}{N_1}\leq\frac{n_0}{N_0}$, perturbing one example in $\Delta \mathcal{I}_R$ with label 1 to $\mathcal{I}_L$ will decrease the information gain. As mentioned in the main text, to decrease the information gain score in Algorithm~\ref{alg:robust_infogain_split}, the adversary needs to perturb examples in $\Delta\mathcal{I}$ such that $\frac{n_0}{N_0}$ and $\frac{n_1}{N_1}$ are close to each other. Algorithm~\ref{alg:split_min} gives an $O(|\Delta \mathcal{I}|)$ method to find $\Delta n_0^*$ and $\Delta n_1^*$, the optimal number of points in $\Delta \mathcal{I}$ with label 0 and 1 to be added to the left.
\begin{algorithm}[tb]
\caption{Finding $\Delta n_0^*$ and $\Delta n_1^*$ to Minimize Information Gain or Gini Impurity}
\label{alg:split_min}
\begin{algorithmic}
\STATE {\bfseries Input:} $N_0$ and $N_1$, number of instances with label 0 and 1. $n_0^o$ and $n_1^o$, number of instances with label 0 and 1 that are certainly on the left.
\STATE {\bfseries Input:} $|\Delta \mathcal{I}\cap \mathcal{I}_0|$ and $|\Delta \mathcal{I}\cap \mathcal{I}_1|$, number of instances with label 0 and 1 that can be perturbed.
\STATE {\bfseries Output:} $\Delta n_0^*$, $\Delta n_1^*$, optimal number of points with label 0 and 1 in $\Delta\mathcal{I}$ to be place on the left.
\STATE$\Delta n_0^*\leftarrow 0,\ \Delta n_1^*\leftarrow 0,\ \text{min\_diff}\leftarrow |\frac{n_0^o}{N_0}-\frac{n_1^o}{N_1}|$;
    \FOR{$\Delta n_0\leftarrow 0$ {\bfseries to} $|\Delta \mathcal{I}\cap \mathcal{I}_0|$}
        \STATE$\text{ceil}\leftarrow \lceil \frac{N_1(n_0^o+\Delta n_0)}{N_0} \rceil-n_1^o$;
        \STATE$\text{floor}\leftarrow \lfloor \frac{N_1(n_0^o+\Delta n_0)}{N_0} \rfloor-n_1^o$;
        \FOR{$\Delta n'_1$ {\bfseries in} $\{\text{ceil},\ \text{floor}\}$}
        \STATE $\Delta n_1\leftarrow \max\{\min\{\Delta n'_1,\ |\Delta \mathcal{I}\cap \mathcal{I}_1|\},\ 0\};$
        \IF {$\text{min\_diff}> |\frac{\Delta n_0+n_0^0}{N_0}-\frac{\Delta n_1+n_1^0}{N_1}|$}
        \STATE $\Delta n_0^*\leftarrow \Delta n_0,\ \Delta n_1^*\leftarrow \Delta n_1,\ \text{min\_diff}\leftarrow|\frac{\Delta n_0+n_0^0}{N_0}-\frac{\Delta n_1+n_1^0}{N_1}|$;
        \ENDIF
        \ENDFOR
    \ENDFOR
    \STATE Return $\Delta n_0^*$ and $\Delta n_1^*$;
    \end{algorithmic}
\end{algorithm}

\section{Gini Impurity Score}
\label{sec:gini}
We also have a theorem for Gini impurity score similar to Theorem~\ref{thm}.
\begin{theorem}
\label{thm:gini}
If $\frac{n_0}{N_0}<\frac{n_1}{N_1}$ and $\frac{n_0+1}{N_0}\leq\frac{n_1}{N_1}$, perturbing one example in $\Delta \mathcal{I}_R$ with label 0 to $\mathcal{I}_L$ will decrease the Gini impurity.
\end{theorem}
\begin{proof}
The Gini impurity score of a split with threshold $\eta$ on feature $j$ is 
\begin{equation}
\begin{split}
Gini&=(1-\frac{|\mathcal{I}_0|^2}{|\mathcal{I}|^2}-\frac{|\mathcal{I}_1|^2}{|\mathcal{I}|^2})\\&-\frac{|\mathcal{I}_L|}{|\mathcal{I}|}(1-\frac{|\mathcal{I}_0\cap\mathcal{I}_L|^2}{|\mathcal{I}_L|^2}-\frac{|\mathcal{I}_1\cap\mathcal{I}_L|^2}{|\mathcal{I}_L|^2})\\&-\frac{|\mathcal{I}_R|}{|\mathcal{I}|}(1-\frac{|\mathcal{I}_0\cap\mathcal{I}_R|^2}{|\mathcal{I}_R|^2}-\frac{|\mathcal{I}_1\cap\mathcal{I}_R|^2}{|\mathcal{I}_R|^2})\\&=C_3[\frac{n_0^2+n_1^2}{n_1+n_0}+\frac{(N_0-n_0)^2+(N_1-n_1)^2}{(N_0+N_1-n_0-n_1)}]+C_4,
\end{split}
\end{equation}
where we use the same notation as in \eqref{eq:IG}. $C_3>0$ and $C_4$ are constants with respect to $n_0$. Taking $n_0$ as a continuous variable, we have
\begin{equation}
\begin{split}
\frac{\partial\ Gini}{\partial n_0}=2C_3\frac{m_1m_0(n_0m_1+n_1m_0+2n_1m_1)}{(n_0+n_1)^2(m_0+m_1)^2}(\frac{n_0}{m_0}-\frac{n_1}{m_1}),
\end{split}
\end{equation}
where $m_0:=N_0-n_0$ and $m_1:=N_1-n_1$. Then $\frac{\partial\ Gini}{\partial n_0}<0$ holds if $\frac{n_0}{m_0}<\frac{n_1}{m_1}$, which is equivalent to $\frac{n_0}{N_0}<\frac{n_1}{N_1}.$ 
\end{proof}
Since the conditions of Theorem~\ref{thm} and Theorem~\ref{thm:gini} are the same, Algorithm~\ref{alg:robust_infogain_split} and Algorithm~\ref{alg:split_min} also work for tree-based models using Gini impurity score. 

\section{Decision Boundaries of Robust and Natural Models}
Figure~\ref{fig:vanilla_decision_boundary} shows the decision boundaries and test accuracy of natural trees as well as robust trees with different $\epsilon$ values on two dimensional synthetic datasets. All trees have depth~5 and we plot training examples in the figure. The results show that the decision boundaries of our robust decision trees are simpler than the decision boundaries in natural decision trees, agreeing with the regularization argument in the main text.
\begin{figure}[htb]\centering
  \includegraphics[width=\linewidth]{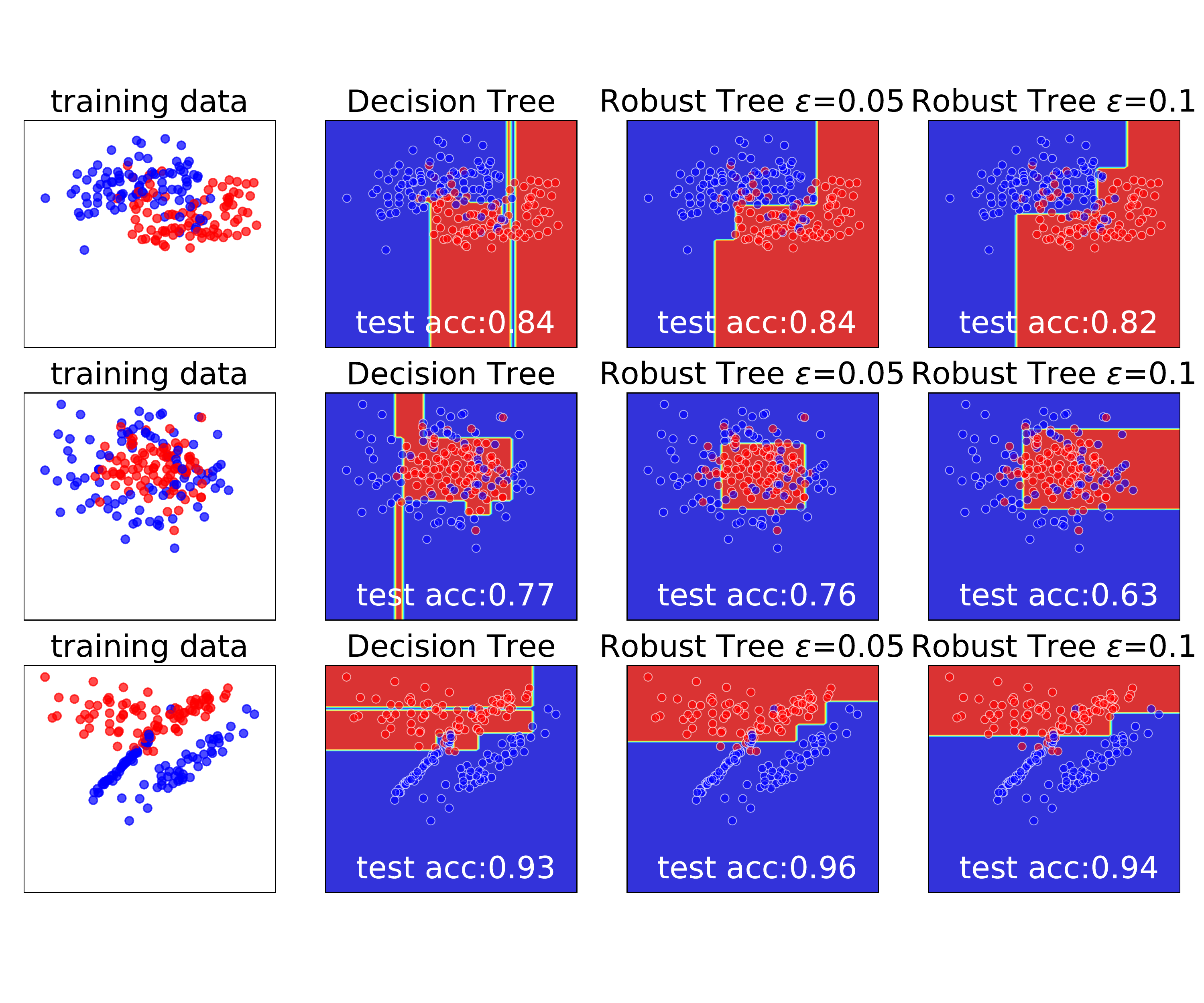}
  \caption{(Best viewed in color) The decision boundaries and test accuracy of natural decision trees and robust decision trees with depth~5 on synthetic datasets with two features. }
  \label{fig:vanilla_decision_boundary}
\end{figure}

\section{Omitted Results on $\ell_1$ and $\ell_2$ distortion}
%In Figures~\ref{fig:robustness_vs_acc_xgboost_mnist_l1_l2},~\ref{fig:robustness_vs_acc_xgboost_fashion_l1_l2} and~\ref{fig:robustness_vs_acc_xgboost_depth_l1_l2} we display the omitted $\ell_1$ and $\ell_2$ distortion vs. accuracy plot of MNIST and Fashion-MNIST datasets. Here we also use Cheng's $\ell_\infty$ attack.

In Tables~\ref{tab:vanilla_robust_acc_table_l1l2} and~\ref{tab:small_robust_acc_table_l1l2} we present the $\ell_1$ and $\ell_2$ distortions of vanilla (information gain based) decision trees and GBDT models obtained by Kantchelian's $\ell_1$ and $\ell_2$ attacks. Again, only small or medium sized binary classification models can be evaluated by Kantchelian's attack. From the results we can see that although our robust decision tree training algorithm is designed for $\ell_\infty$ perturbations, it can also improve models $\ell_1$ and $\ell_2$ robustness significantly.

\section{Omitted Results on Models with Different Number of Trees}
Figure~\ref{fig:robustness_vs_acc_xgboost_fashion} shows the $\ell_\infty$ distortion and accuracy of Fashion-MNIST GBDT models with different number of trees. In Table~\ref{tab:big_robust_acc_table} we present the test accuracy and $\ell_\infty$ distortion of models with different number of trees obtained by Cheng's $\ell_\infty$ attack. For each dataset, models are generated during a single boosting run. We can see that the robustness of robustly trained models consistently outperforms that of natural models with the same number of trees. Another interesting finding is that for MNIST and Fashion-MNIST datasets in Figures~\ref{fig:robustness_vs_acc_xgboost_mnist} (in the main text) and~\ref{fig:robustness_vs_acc_xgboost_fashion}, models with more trees are generally more robust. This may not be true in other datasets; for example, results from Table~\ref{tab:big_robust_acc_table} in the Appendix shows that on some other datasets, the natural GBDT models lose robustness when more trees are added.
\begin{figure}[htb]\centering
\includegraphics[width=1.0\linewidth]{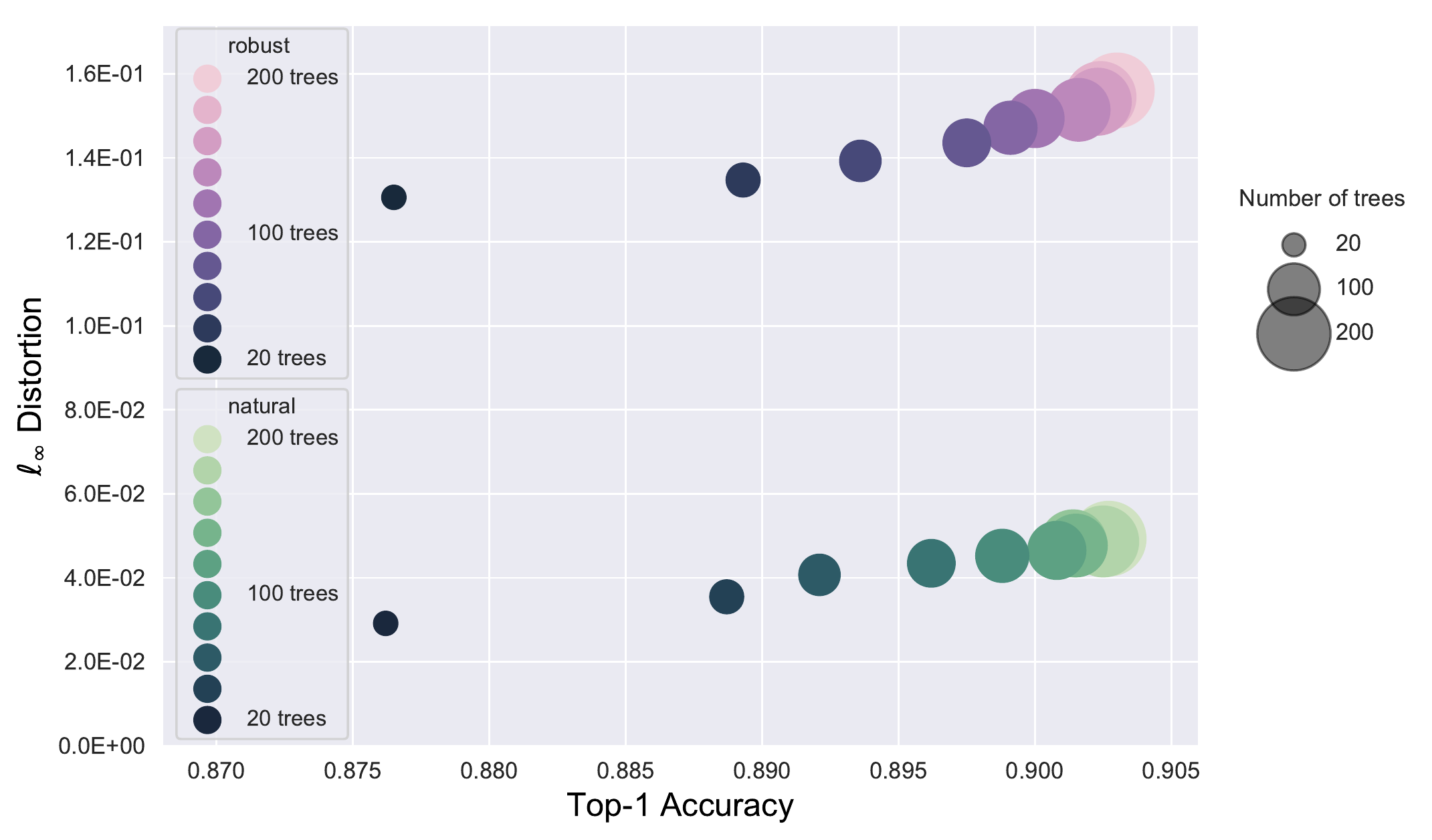}
  
  \caption{(Best viewed in color) $\ell_\infty$ distortion vs. classification accuracy of GBDT models on Fashion-MNIST datasets with different numbers of trees (circle size). The adversarial examples are found by Cheng's $\ell_\infty$ attack. The robust training parameter $\epsilon=0.1$ for Fashion-MNIST. With robust training (purple) the distortion needed to fool a model increases dramatically with less than $1\%$ accuracy loss.} 
  \label{fig:robustness_vs_acc_xgboost_fashion}
\end{figure}

\section{Reducing Depth Does Not Improve Robustness}
\label{sec:depth}
One might hope that one can simply reduce the depth of trees to improve robustness since shallower trees provide stronger regularization effects. Unfortunately, this is not true. As demonstrated in Figure~\ref{fig:robustness_vs_acc_xgboost_depth}, the robustness of naturally trained GBDT models are much worse when compared to robust models, no matter how shallow they are or how many trees are in the ensemble. Also, when the number of trees in the ensemble model is limited, reducing tree depth will significantly lower the model accuracy.
\begin{figure}[htb]\centering
\includegraphics[width=1.0\linewidth]{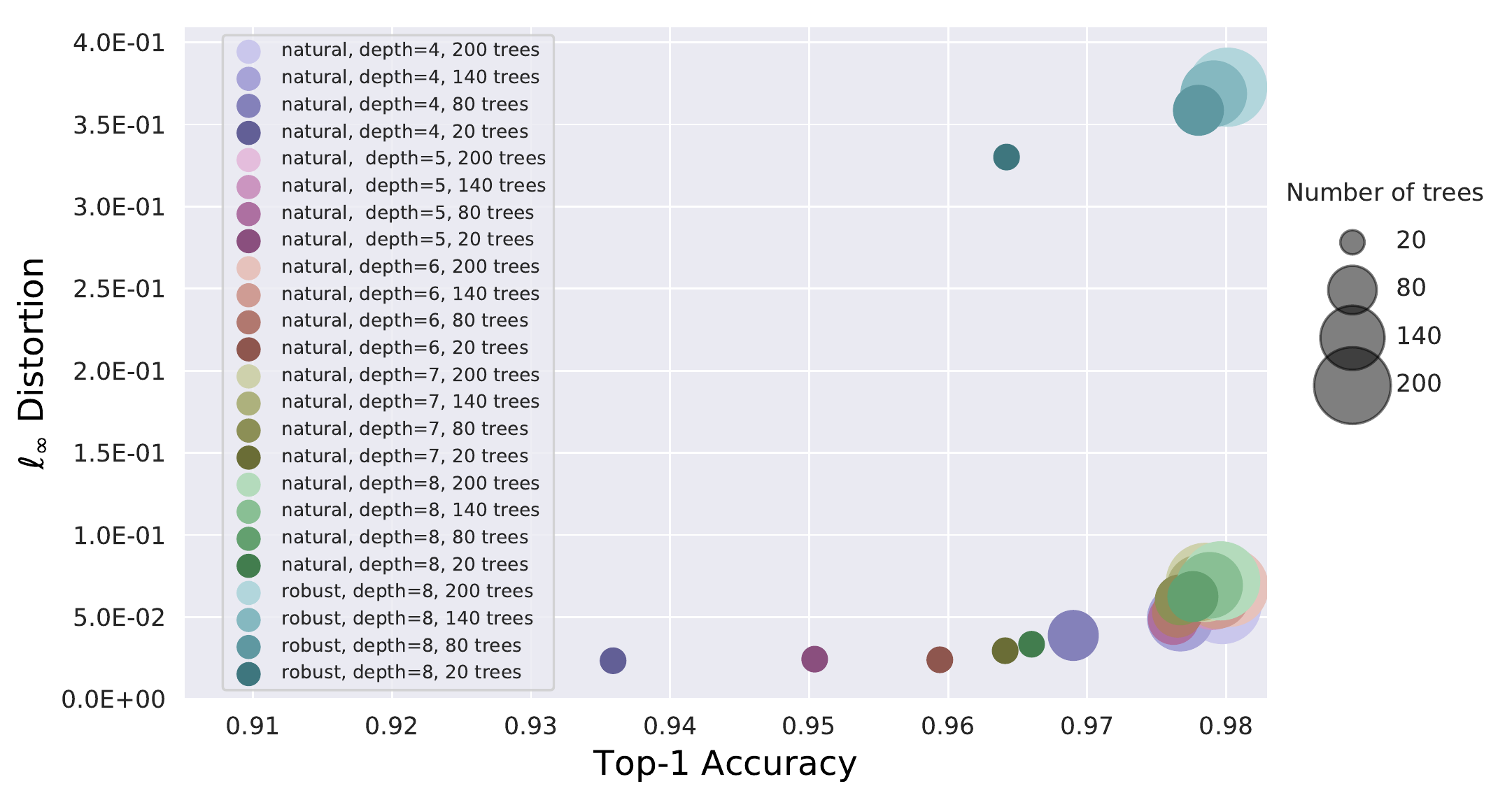}
  
  \caption{(Best viewed in color) Robustness vs. classification accuracy plot of GBDT models on MNIST dataset with different depth and different numbers of trees. The adversarial examples are found by Cheng's $\ell_\infty$ attack. The robust training parameter $\epsilon=0.3$. Reducing the model depth cannot improve robustness effectively compared to our proposed robust training procedure.}%We display the $\ell_1$ and $\ell_2$ distortion results in Figure~\ref{fig:robustness_vs_acc_xgboost_depth_l1_l2} in the Appendix.}
  \label{fig:robustness_vs_acc_xgboost_depth}
\end{figure}

\section{Random Forest Model Results}
\label{sec:rf}
We test our robust training framework on random forest (RF) models and our results are in Table~\ref{tab:rf_table}. In these experiments we build random forest models with 0.5 data sampling rate and 0.5 feature sampling rate.  We test the robust and natural random forest model on three datasets and in each dataset, we tested 100 points using Cheng's and Kantchelian's $\ell_\infty$ attacks. From the results we can see that our robust decision tree training framework can also significantly improve random forest model robustness. 

\section{More MNIST and Fashion-MNIST Adversarial Examples}
In Figure~\ref{fig:appendix_demo} we present more adversarial examples for MNIST and Fashion-MNIST datasets using GBDT models.

\begin{table*}[htb]
\begin{center}
\scalebox{0.7}{
\setlength\tabcolsep{1.5pt}
\begin{tabular}{c|c|c|c|c|c|c|c|c|c|c|c|c|c}

\Xhline{5\arrayrulewidth}
    
    \multirow{2}*{Dataset}&training&test&\# of&\# of &\multirow{2}*{robust $\epsilon$}& \multicolumn{2}{c|}{depth}&\multicolumn{2}{c|}{test acc.}&\multicolumn{2}{c|}{\thead{avg. $\ell_1$ dist.\\by Kantchelian's $\ell_1$ attack}}&\multicolumn{2}{c}{\thead{avg. $\ell_2$ dist.\\by Kantchelian's $\ell_2$ attack}}\\
    &set size&set size&features&classes&&robust&natural&robust&natural&robust&natural&robust&natural\\\Xhline{3\arrayrulewidth}
    
    breast-cancer &546&137&10&2&0.3&5&5&.948&.942&\textbf{.534}&.270&\textbf{.504}&.209\\
    
    diabetes&614&154&8&2&0.2&5&5&.688&.747&\textbf{.204}&.075&\textbf{.204}&.065\\
    
    ionosphere&281&70&34&2&0.2&4&4&.986&.929&\textbf{.358}&.127&\textbf{.358}&.106\\
    
    \Xhline{5\arrayrulewidth}
    
\end{tabular}
}
\caption{The test accuracy and robustness of information gain based single decision tree models. The robustness is evaluated by the average $\ell_1$ and $\ell_2$ distortions of adversarial examples found by Kantchelian's $\ell_1$ and $\ell_2$ attacks. Average $\ell_\infty$ distortions of robust decision tree models found by the two attack methods are consistently larger than those of the naturally trained ones.}
\label{tab:vanilla_robust_acc_table_l1l2}
\end{center}
\end{table*}

\begin{table*}[htb!]
\begin{center}
\scalebox{0.7}{
\setlength\tabcolsep{1.5pt}
\begin{tabular}{c|c|c|c|c|c|c|c|c|c|c|c|c|c|c|c|c}

\Xhline{5\arrayrulewidth}
    
    \multirow{2}*{Dataset}&training&test&\# of&\# of &\# of &robust& \multicolumn{2}{c|}{depth}&\multicolumn{2}{c|}{test acc.}&\multicolumn{2}{c|}{\thead{avg. $\ell_1$ dist.\\by Kantchelian's $\ell_1$ attack}}&dist.&\multicolumn{2}{c|}{\thead{avg. $\ell_2$ dist.\\by Kantchelian's $\ell_2$ attack}}&dist.\\
    &set size&set size&features&classes&trees& $\epsilon$&robust&natural&robust&natural&robust&natural&improv.&robust&natural&improv.\\\Xhline{3\arrayrulewidth}
    
    breast-cancer &546&137&10&2&4&0.3&8&6&.978&.964&\textbf{.488}&.328&\textbf{1.49X}&\textbf{.431}&.251&\textbf{1.72X}\\
    
    cod-rna&59,535&271,617&8&2&80&0.2&5&4&.880&.965&\textbf{.065}&.059&\textbf{1.10X}&\textbf{.062}&.047&\textbf{1.32X}\\
    diabetes&614&154&8&2&20&0.2&5&5&.786&.773&\textbf{.150}&.081&\textbf{1.85X}&\textbf{.135}&.059&\textbf{2.29X}\\

    ijcnn1&49,990&91,701&22&2&60&0.1&8&8&.959&.980&\textbf{.057}&.051&\textbf{1.12X}&\textbf{.048}&.042&\textbf{1.14X}\\
    MNIST 2 vs. 6&11,876&1,990&784&2&1000&0.3&6&4&.997&.998&\textbf{1.843}&.721&\textbf{2.56X}&\textbf{.781}&.182&\textbf{4.29X}\\%\textbf{.015}&&\\
    \Xhline{5\arrayrulewidth}

\end{tabular}
}
\caption{The test accuracy and robustness of GBDT models. Average $\ell_1$ and $\ell_2$ distortions of robust GBDT models are consistently larger than those of the naturally trained models. The robustness is evaluated by the average $\ell_1$ and $\ell_2$ distortions of adversarial examples found by Kantchelian's $\ell_1$ and $\ell_2$ attacks.}
\label{tab:small_robust_acc_table_l1l2}
\end{center}
\end{table*}

\begin{table*}[htb]
\begin{center}
\scalebox{0.7}{
\setlength\tabcolsep{1.5pt}
\begin{tabular}{c|c|c|c|c|c|c|c|c|c|c|c|c|c|c|c|c}

\Xhline{5\arrayrulewidth}
    
    \multirow{2}*{Dataset}&training&test&\# of&\# of &\# of &robust& \multicolumn{2}{c|}{depth}&\multicolumn{2}{c|}{test acc.}&\multicolumn{2}{c|}{\thead{avg. $\ell_\infty$ dist.\\by Cheng's $\ell_\infty$ attack}}&dist.&\multicolumn{2}{c|}{\thead{avg. $\ell_\infty$ dist.\\by Kantchelian's $\ell_\infty$ attack}}&dist.\\
    &set size&set size&features&classes&trees& $\epsilon$&robust&natural&robust&natural&robust&natural&improv.&robust&natural&improv.\\\Xhline{3\arrayrulewidth}
    breast-cancer &546&137&10&2&60&0.3&8&6&.993&.993&\textbf{.406}&.297&\textbf{1.37X}&\textbf{.396}&.244&\textbf{1.62X}\\
    diabetes&614&154&8&2&60&0.2&5&5&.753&.760&\textbf{.185}&.093&\textbf{1.99X}&\textbf{.154}&.072&\textbf{2.14X}\\
    MNIST 2 vs. 6&11,876&1,990&784&2&1000&0.3&6&4&.986&.983&\textbf{.445}&.180&\textbf{2.47X}&\textbf{.341}&.121&\textbf{2.82X}\\
    \Xhline{5\arrayrulewidth}

\end{tabular}
}
\caption{The test accuracy and robustness of random forest models. Average $\ell_\infty$ distortion of our robust random forest models are consistently larger than those of the naturally trained models. The robustness is evaluated by the average $\ell_\infty$ distortion of adversarial examples found by Cheng's and Kantchelian's attacks.}
\label{tab:rf_table}
\end{center}
\end{table*}

\begin{table*}[htb!]
\begin{center}
\scalebox{0.65}{
\setlength\tabcolsep{1.5pt}
\begin{tabular}{c?c?c?c?c?c|c?c|c?c|c?c|c?c|c?c|c?c|c?c|c?c|c?c|c}
    \Xhline{5\arrayrulewidth}
    
    \multirow{4}{*}{\thead{breast-cancer (2)\\ $\epsilon=0.3$\\$\text{depth}_r=8,\ \text{depth}_n=6$}}&train&test&feat.&\# of trees&\multicolumn{2}{c?}{1}&\multicolumn{2}{c?}{ 2}&\multicolumn{2}{c?}{ 3}&\multicolumn{2}{c?}{4}&\multicolumn{2}{c?}{5}&\multicolumn{2}{c?}{6}&\multicolumn{2}{c?}{ 7}&\multicolumn{2}{c?}{8}&\multicolumn{2}{c?}{ 9}&\multicolumn{2}{c}{10}\\\cline{2-25}
    &&&&model&rob.&nat.&rob.&nat.&rob.&nat.&rob.&nat.&rob.&nat.&rob.&nat.&rob.&nat.&rob.&nat.&rob.&nat.&rob.&nat.\\\cline{5-25}
    &546&137&10&tst. acc.&.985&.942&.971&.964&.978&.956&.978&.964&.985&.964&.985&.964&.985&.971&.993&.971&.993&.971&1.00&.971\\\cline{5-25}
   &&&&$\ell_\infty$ dist.&\textbf{.383}&.215&\textbf{.396}&.229&\textbf{.411}&.216&\textbf{.411}&.215&\textbf{.406}&.226&\textbf{.407}&.229&\textbf{.406}&.248&\textbf{.439}&.234&\textbf{.439}&.238&\textbf{.437}&.241\\
    \Xhline{5\arrayrulewidth}

    \multirow{4}{*}{\thead{covtype (7)\\$\epsilon=0.2$\\$\text{depth}_r=\text{depth}_n=8$}}&train&test&feat.&\# of trees&\multicolumn{2}{c?}{20}&\multicolumn{2}{c?}{ 40}&\multicolumn{2}{c?}{ 60}&\multicolumn{2}{c?}{80}&\multicolumn{2}{c?}{100}&\multicolumn{2}{c?}{ 120}&\multicolumn{2}{c?}{ 140}&\multicolumn{2}{c?}{160}&\multicolumn{2}{c?}{ 180}&\multicolumn{2}{c}{200}\\\cline{2-25}
    &&&&model&rob.&nat.&rob.&nat.&rob.&nat.&rob.&nat.&rob.&nat.&rob.&nat.&rob.&nat.&rob.&nat.&rob.&nat.&rob.&nat.\\\cline{5-25}
    &400,000&181,000&54&tst. acc.&.775&.828&.809&.850&.832&.865&.847&.877&.858&.891&.867&.902&.875&.912&.882&.921&.889&.926&.894&.930\\\cline{5-25}
   &&&&$\ell_\infty$ dist.&\textbf{.125}&.066&\textbf{.103}&.064&\textbf{.087}&.062&\textbf{.081}&.061&\textbf{.079}&.060&\textbf{.077}&.059&\textbf{.077}&.058&\textbf{.075}&.056&\textbf{.075}&.056&\textbf{.073}&.055\\
    \Xhline{5\arrayrulewidth}
    
    \multirow{4}{*}{\thead{cod-rna (2)\\ $\epsilon=0.2$\\$\text{depth}_r=5,\ \text{depth}_n=4$}}&train&test&feat.&\# of trees&\multicolumn{2}{c?}{20}&\multicolumn{2}{c?}{ 40}&\multicolumn{2}{c?}{ 60}&\multicolumn{2}{c?}{80}&\multicolumn{2}{c?}{100}&\multicolumn{2}{c?}{ 120}&\multicolumn{2}{c?}{ 140}&\multicolumn{2}{c?}{160}&\multicolumn{2}{c?}{ 180}&\multicolumn{2}{c}{200}\\\cline{2-25}
    &&&&model&rob.&nat.&rob.&nat.&rob.&nat.&rob.&nat.&rob.&nat.&rob.&nat.&rob.&nat.&rob.&nat.&rob.&nat.&rob.&nat.\\\cline{5-25}
    &59,535&271,617&8&tst. acc.&.810&.947&.861&.959&.874&.963&.880&.965&.892&.966&.900&.967&.903&.967&.915&.967&.922&.967&.925&.968\\\cline{5-25}
   &&&&$\ell_\infty$ dist.&\textbf{.077}&.057&\textbf{.066}&.055&\textbf{.063}&.054&\textbf{.062}&.053&\textbf{.059}&.053&\textbf{.057}&.052&\textbf{.056}&.052&\textbf{.056}&.052&\textbf{.056}&.052&\textbf{.058}&.052\\
    \Xhline{5\arrayrulewidth}
    
    \multirow{4}{*}{\thead{diabetes (2)\\ $\epsilon=0.2$\\$\text{depth}_r=\text{depth}_n=5$}}&train&test&feat.&\# of trees&\multicolumn{2}{c?}{2}&\multicolumn{2}{c?}{ 4}&\multicolumn{2}{c?}{ 6}&\multicolumn{2}{c?}{8}&\multicolumn{2}{c?}{10}&\multicolumn{2}{c?}{ 12}&\multicolumn{2}{c?}{ 14}&\multicolumn{2}{c?}{16}&\multicolumn{2}{c?}{ 18}&\multicolumn{2}{c}{20}\\\cline{2-25}
    &&&&model&rob.&nat.&rob.&nat.&rob.&nat.&rob.&nat.&rob.&nat.&rob.&nat.&rob.&nat.&rob.&nat.&rob.&nat.&rob.&nat.\\\cline{5-25}
    &614&154&8&tst. acc.&.760&.753&.760&.753&.766&.753&.773&.753&.773&.734&.779&.727&.779&.747&.779&.760&.779&.773&.786&.773\\\cline{5-25}
   &&&&$\ell_\infty$ dist.&\textbf{.163}&.066&\textbf{.163}&.065&\textbf{.154}&.071&\textbf{.151}&.071&\textbf{.152}&.073&\textbf{.148}&.072&\textbf{.146}&.067&\textbf{.144}&.062&\textbf{.138}&.062&\textbf{.139}&.060\\
    \Xhline{5\arrayrulewidth}
    
    \multirow{4}{*}{\thead{Fashion-MNIST (10) \\ $\epsilon=0.1$\\$\text{depth}_r=\text{depth}_n=8$}}&train&test&feat.&\# of trees&\multicolumn{2}{c?}{20}&\multicolumn{2}{c?}{ 40}&\multicolumn{2}{c?}{ 60}&\multicolumn{2}{c?}{80}&\multicolumn{2}{c?}{100}&\multicolumn{2}{c?}{ 120}&\multicolumn{2}{c?}{ 140}&\multicolumn{2}{c?}{160}&\multicolumn{2}{c?}{ 180}&\multicolumn{2}{c}{200}\\\cline{2-25}
    &&&&model&rob.&nat.&rob.&nat.&rob.&nat.&rob.&nat.&rob.&nat.&rob.&nat.&rob.&nat.&rob.&nat.&rob.&nat.&rob.&nat.\\\cline{5-25}
    &60,000&10,000&784&tst. acc.&.877&.876&.889&.889&.894&.892&.898&.896&.899&.899&.900&.901&.902&.902&.902&.901&.902&.903&.903&.903\\\cline{5-25}
   &&&&$\ell_\infty$ dist.&\textbf{.131}&.029&\textbf{.135}&.035&\textbf{.139}&.041&\textbf{.144}&.043&\textbf{.147}&.045&\textbf{.149}&.047&\textbf{.151}&.048&\textbf{.153}&.048&\textbf{.154}&.049&\textbf{.156}&.049\\
    \Xhline{5\arrayrulewidth}
    
    \multirow{4}{*}{\thead{HIGGS (2)\\ $\epsilon=0.05$\\$\text{depth}_r=\text{depth}_n=8$}}&train&test&feat.&\# of trees&\multicolumn{2}{c?}{50}&\multicolumn{2}{c?}{100}&\multicolumn{2}{c?}{ 150}&\multicolumn{2}{c?}{200}&\multicolumn{2}{c?}{250}&\multicolumn{2}{c?}{ 300}&\multicolumn{2}{c?}{350}&\multicolumn{2}{c?}{400}&\multicolumn{2}{c?}{ 450}&\multicolumn{2}{c}{500}\\\cline{2-25}
    &&&&model&rob.&nat.&rob.&nat.&rob.&nat.&rob.&nat.&rob.&nat.&rob.&nat.&rob.&nat.&rob.&nat.&rob.&nat.&rob.&nat.\\\cline{5-25}
    &10,500,000&500,000&28&tst. acc.&.676&.747&.688&.753&.700&.755&.702&.758&.705&.759&.709&.760&.711&.762&.712&.764&.716&.763&.718&.764\\\cline{5-25}
   &&&&$\ell_\infty$ dist.&\textbf{.023}&.013&\textbf{.023}&.014&\textbf{.022}&.014&\textbf{.022}&.014&\textbf{.022}&.014&\textbf{.022}&.014&\textbf{.021}&.015&\textbf{.021}&.015&\textbf{.021}&.015&\textbf{.021}&.015\\
    \Xhline{5\arrayrulewidth}
    
    \multirow{4}{*}{\thead{ijcnn1 (2)\\ $\epsilon=0.1$\\$\text{depth}_r=\text{depth}_n=8$}}&train&test&feat.&\# of trees&\multicolumn{2}{c?}{10}&\multicolumn{2}{c?}{ 20}&\multicolumn{2}{c?}{ 30}&\multicolumn{2}{c?}{40}&\multicolumn{2}{c?}{50}&\multicolumn{2}{c?}{60}&\multicolumn{2}{c?}{ 70}&\multicolumn{2}{c?}{80}&\multicolumn{2}{c?}{90 }&\multicolumn{2}{c}{100}\\\cline{2-25}
    &&&&model&rob.&nat.&rob.&nat.&rob.&nat.&rob.&nat.&rob.&nat.&rob.&nat.&rob.&nat.&rob.&nat.&rob.&nat.&rob.&nat.\\\cline{5-25}
    &49,990&91,701&22&tst. acc.&.933&.973&.942&.977&.947&.977&.952&.979&.958&.980&.959&.980&.962&.980&.964&.980&.967&.980&.968&.980\\\cline{5-25}
   &&&&$\ell_\infty$ dist.&\textbf{.065}&.048&\textbf{.061}&.047&\textbf{.058}&.048&\textbf{.057}&.047&\textbf{.054}&.048&\textbf{.054}&.047&\textbf{.054}&.047&\textbf{.053}&.047&\textbf{.052}&.047&\textbf{.052}&.047\\
    \Xhline{5\arrayrulewidth}
    
    %\multirow{4}{*}{\thead{libsvm MNIST (10)\\ $\epsilon=0.3$\\$\text{depth}_r=\text{depth}_n=8$}}&train&test&feat.&\# of trees&\multicolumn{2}{c?}{20}&\multicolumn{2}{c?}{ 40}&\multicolumn{2}{c?}{ 60}&\multicolumn{2}{c?}{80}&\multicolumn{2}{c?}{100}&\multicolumn{2}{c?}{ 120}&\multicolumn{2}{c?}{ 140}&\multicolumn{2}{c?}{160}&\multicolumn{2}{c?}{ 180}&\multicolumn{2}{c}{200}\\\cline{2-25}
    %&&&&model&rob.&nat.&rob.&nat.&rob.&nat.&rob.&nat.&rob.&nat.&rob.&nat.&rob.&nat.&rob.&nat.&rob.&nat.&rob.&nat.\\\cline{5-25}
    %&$60,000$&$10,000$&780&tst. acc.&.962&.966&.973&.975&.976&.977&.979&.978&.979&.978&.979&.979&.980&.979&.980&.979&.981&.979&.981&.980\\\cline{5-25}
   %&&&&$\ell_\infty$ dist.&.329&.033&.343&.049&.354&.057&.360&.062&.364&.065&.367&.067&.369&.069&.371&.071&.372&.072&.374&.072\\
    %\Xhline{5\arrayrulewidth}
    
    \multirow{4}{*}{\thead{MNIST (10)\\ $\epsilon=0.3$\\$\text{depth}_r=\text{depth}_n=8$}}&train&test&feat.&\# of trees&\multicolumn{2}{c?}{20}&\multicolumn{2}{c?}{ 40}&\multicolumn{2}{c?}{ 60}&\multicolumn{2}{c?}{80}&\multicolumn{2}{c?}{100}&\multicolumn{2}{c?}{ 120}&\multicolumn{2}{c?}{ 140}&\multicolumn{2}{c?}{160}&\multicolumn{2}{c?}{ 180}&\multicolumn{2}{c}{200}\\\cline{2-25}
    &&&&model&rob.&nat.&rob.&nat.&rob.&nat.&rob.&nat.&rob.&nat.&rob.&nat.&rob.&nat.&rob.&nat.&rob.&nat.&rob.&nat.\\\cline{5-25}
    &$60,000$&$10,000$&784&tst. acc.&.964&.966&.973&.975&.977&.977&.978&.978&.978&.978&.979&.979&.979&.979&.980&.979&.980&.979&.980&.980\\\cline{5-25}
   &&&&$\ell_\infty$ dist.&\textbf{.330}&.033&\textbf{.343}&.049&\textbf{.352}&.057&\textbf{.359}&.062&\textbf{.363}&.065&\textbf{.367}&.067&\textbf{.369}&.069&\textbf{.370}&.071&\textbf{.371}&.072&\textbf{.373}&.072\\
    \Xhline{5\arrayrulewidth}
    
    \multirow{4}{*}{\thead{Sensorless (11)\\ $\epsilon=0.05$\\$\text{depth}_r=\text{depth}_n=6$}}&train&test&feat.&\# of trees&\multicolumn{2}{c?}{3}&\multicolumn{2}{c?}{ 6}&\multicolumn{2}{c?}{ 9}&\multicolumn{2}{c?}{12}&\multicolumn{2}{c?}{15}&\multicolumn{2}{c?}{18}&\multicolumn{2}{c?}{ 21}&\multicolumn{2}{c?}{24}&\multicolumn{2}{c?}{ 27}&\multicolumn{2}{c}{30}\\\cline{2-25}
    &&&&model&rob.&nat.&rob.&nat.&rob.&nat.&rob.&nat.&rob.&nat.&rob.&nat.&rob.&nat.&rob.&nat.&rob.&nat.&rob.&nat.\\\cline{5-25}
    &48,509&10,000&48&tst. acc.&.834&.977&.867&.983&.902&.987&.923&.991&.945&.992&.958&.994&.966&.996&.971&.996&.974&.997&.978&.997\\\cline{5-25}
   &&&&$\ell_\infty$ dist.&\textbf{.037}&.022&\textbf{.036}&.022&\textbf{.035}&.023&\textbf{.035}&.023&\textbf{.035}&.023&\textbf{.035}&.023&\textbf{.035}&.023&\textbf{.035}&.023&\textbf{.035}&.023&\textbf{.035}&.023\\
    \Xhline{5\arrayrulewidth}
    
    \multirow{4}{*}{\thead{webspam (2)\\ $\epsilon=0.05$\\$\text{depth}_r=\text{depth}_n=8$}}&train&test&feat.&\# of trees&\multicolumn{2}{c?}{10}&\multicolumn{2}{c?}{ 20}&\multicolumn{2}{c?}{ 30}&\multicolumn{2}{c?}{40}&\multicolumn{2}{c?}{50}&\multicolumn{2}{c?}{ 60}&\multicolumn{2}{c?}{ 70}&\multicolumn{2}{c?}{80}&\multicolumn{2}{c?}{ 90}&\multicolumn{2}{c}{100}\\\cline{2-25}
    &&&&model&rob.&nat.&rob.&nat.&rob.&nat.&rob.&nat.&rob.&nat.&rob.&nat.&rob.&nat.&rob.&nat.&rob.&nat.&rob.&nat.\\\cline{5-25}
    &$300,000$&$50,000$&254&tst. acc.&.950&.976&.964&.983&.970&.986&.973&.989&.976&.990&.978&.990&.980&.991&.981&.991&.982&.992&.983&.992\\\cline{5-25}
   &&&&$\ell_\infty$ dist.&\textbf{.049}&.010&\textbf{.048}&.015&\textbf{.049}&.019&\textbf{.049}&.021&\textbf{.049}&.023&\textbf{.049}&.024&\textbf{.049}&.024&\textbf{.049}&.024&\textbf{.048}&.024&\textbf{.049}&.024\\
    \Xhline{5\arrayrulewidth}

\end{tabular}
}
\caption{The test accuracy and robustness of GBDT models. Here $\text{depth}_n$ is the depth of natural trees and $\text{depth}_r$ is the depth of robust trees. Robustness is evaluated by the average $\ell_\infty$ distortion of adversarial examples found by Cheng's attack~\citep{cheng2018queryefficient}. The number in the parentheses after each dataset name is the number of classes. Models are generated during a single boosting run. We can see that the robustness of our robust models consistently outperforms that of natural models with the same number of trees.}
\label{tab:big_robust_acc_table}
\end{center}
\end{table*}

\begin{figure*}[htb]
\centering
\begin{tabular}{c|c|c?c|c|c}
\textbf{Original}&\thead{Adversarial of \\ nat. GBDT}&\thead{Adversarial of \\ rob. GBDT}&\textbf{Original}&\thead{Adversarial of \\ nat. GBDT}&\thead{Adversarial of \\ rob. GBDT}\\
\begin{subfigure}[t]{0.12\textwidth}
\centering
  \includegraphics[width=\linewidth]{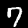}
  \caption{pred.=7~~~~~~~~~~~~~~~~~~~~~~~~}
\end{subfigure}&
\begin{subfigure}[t]{0.12\textwidth}
\centering
  \includegraphics[width=\linewidth]{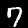}
  \caption{\tabular[t]{@{}l@{}}$\ell_\infty$ dist.$=0.002$\\pred.=9\endtabular}
\end{subfigure}&
\begin{subfigure}[t]{0.12\textwidth}
\centering
  \includegraphics[width=\linewidth]{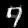}
  \caption{\tabular[t]{@{}l@{}}$\ell_\infty$ dist.$=0.305$\\pred.=9\endtabular}
\end{subfigure}&
\begin{subfigure}[t]{0.12\textwidth}
\centering
  \includegraphics[width=\linewidth]{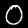}
  \caption{pred.=0~~~~~~~~~~~~~~~~~~~~~~~~}
\end{subfigure}&
\begin{subfigure}[t]{0.12\textwidth}
\centering
  \includegraphics[width=\linewidth]{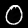}
  \caption{\tabular[t]{@{}l@{}}$\ell_\infty$ dist.$=0.018$\\pred.=8\endtabular}
\end{subfigure}&
\begin{subfigure}[t]{0.12\textwidth}
\centering
  \includegraphics[width=\linewidth]{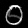}
  \caption{\tabular[t]{@{}l@{}}$\ell_\infty$ dist.$=0.327$\\pred.=5\endtabular}
\end{subfigure}\\
\begin{subfigure}[t]{0.12\textwidth}
\centering
  \includegraphics[width=\linewidth]{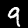}
  \caption{pred.=9~~~~~~~~~~~~~~~~~~~~~~~~}
\end{subfigure}&
\begin{subfigure}[t]{0.12\textwidth}
\centering
  \includegraphics[width=\linewidth]{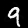}
  \caption{\tabular[t]{@{}l@{}}$\ell_\infty$ dist.$=0.025$\\pred.=4\endtabular}
\end{subfigure}&
\begin{subfigure}[t]{0.12\textwidth}
\centering
  \includegraphics[width=\linewidth]{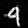}
  \caption{\tabular[t]{@{}l@{}}$\ell_\infty$ dist.$=0.402$\\pred.=4\endtabular}
\end{subfigure}&
\begin{subfigure}[t]{0.12\textwidth}
\centering
  \includegraphics[width=\linewidth]{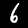}
  \caption{pred.=6~~~~~~~~~~~~~~~~~~~~~~~~}
\end{subfigure}&
\begin{subfigure}[t]{0.12\textwidth}
\centering
  \includegraphics[width=\linewidth]{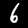}
  \caption{\tabular[t]{@{}l@{}}$\ell_\infty$ dist.$=0.014$\\pred.=8\endtabular}
\end{subfigure}&
\begin{subfigure}[t]{0.12\textwidth}
\centering
  \includegraphics[width=\linewidth]{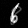}
  \caption{\tabular[t]{@{}l@{}}$\ell_\infty$ dist.$=0.329$\\pred.=8\endtabular}
\end{subfigure}\\
\begin{subfigure}[t]{0.12\textwidth}
\centering
  \includegraphics[width=\linewidth]{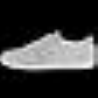}
  \caption{pred.=``Sneaker''}
\end{subfigure}&
\begin{subfigure}[t]{0.12\textwidth}
\centering
  \includegraphics[width=\linewidth]{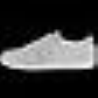}
  \caption{\tabular[t]{@{}l@{}}$\ell_\infty$ dist.$=0.025$\\pred.=``Bag''\endtabular}
\end{subfigure}&
\begin{subfigure}[t]{0.12\textwidth}
\centering
  \includegraphics[width=\linewidth]{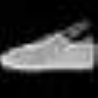}
  \caption{\tabular[t]{@{}l@{}}$\ell_\infty$ dist.$=0.482$\\pred.=``Sandal''\endtabular}
\end{subfigure}&
\begin{subfigure}[t]{0.12\textwidth}
\centering
  \includegraphics[width=\linewidth]{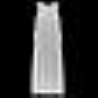}
  \caption{pred.=``Dress''}
\end{subfigure}&
\begin{subfigure}[t]{0.12\textwidth}
\centering
  \includegraphics[width=\linewidth]{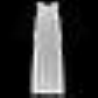}
  \caption{\tabular[t]{@{}l@{}}$\ell_\infty$ dist.$=0.024$\\pred.=``T-shirt/top''\endtabular}
\end{subfigure}&
\begin{subfigure}[t]{0.12\textwidth}
\centering
  \includegraphics[width=\linewidth]{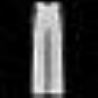}
  \caption{\tabular[t]{@{}l@{}}$\ell_\infty$ dist.$=0.340$\\pred.=``Trouser''\endtabular}
\end{subfigure}
\\
\begin{subfigure}[t]{0.12\textwidth}
\centering
  \includegraphics[width=\linewidth]{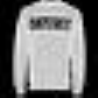}
  \caption{pred.=``Pullover''}
\end{subfigure}&
\begin{subfigure}[t]{0.12\textwidth}
\centering
  \includegraphics[width=\linewidth]{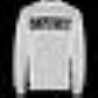}
  \caption{\tabular[t]{@{}l@{}}$\ell_\infty$ dist.$=0.017$\\pred.=``Bag''\endtabular}
\end{subfigure}&
\begin{subfigure}[t]{0.12\textwidth}
\centering
  \includegraphics[width=\linewidth]{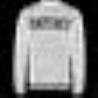}
  \caption{\tabular[t]{@{}l@{}}$\ell_\infty$ dist.$=0.347$\\pred.=``Coat''\endtabular}
\end{subfigure}&
\begin{subfigure}[t]{0.12\textwidth}
\centering
  \includegraphics[width=\linewidth]{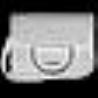}
  \caption{pred.=``Bag''~~~~~~~~~}
\end{subfigure}&
\begin{subfigure}[t]{0.12\textwidth}
\centering
  \includegraphics[width=\linewidth]{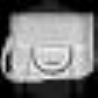}
  \caption{\tabular[t]{@{}l@{}}$\ell_\infty$ dist.$=0.033$\\pred.=``Shirt''\endtabular}
\end{subfigure}&
\begin{subfigure}[t]{0.12\textwidth}
\centering
  \includegraphics[width=\linewidth]{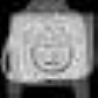}
  \caption{\tabular[t]{@{}l@{}}$\ell_\infty$ dist.$=0.441$\\pred.=``Coat''\endtabular}
\end{subfigure}
\end{tabular}
\caption{MNIST and Fashion-MNIST examples and their adversarial examples found using the untargeted Cheng's $\ell_\infty$ attack~\citep{cheng2018queryefficient} on 200-tree gradient boosted decision tree (GBDT) models trained using XGBoost with
depth=8. For both MNIST and Fashion-MNIST robust models, we use $\epsilon=0.3$. }
\label{fig:appendix_demo}
\end{figure*}

\end{document}